%% file: ActiveLearningGM.tex
\documentclass[twoside]{article}
\usepackage[accepted]{aistats2014}

\usepackage{amssymb,pstricks,amsthm}
\usepackage{graphicx}
\usepackage{cite}
\usepackage{lipsum}
\usepackage{enumerate}
\usepackage{paralist}
\usepackage{tikz}
\usetikzlibrary{shapes,decorations}
\usepgflibrary{shapes.geometric}
\usepackage{amsmath}
\usepackage{subfigure}
\usepackage{url}
\usepackage{float}
\usepackage{arydshln}
\usepackage{wrapfig}
\usepackage{multirow}
\usepackage{algorithmic}
\usepackage[ruled, vlined]{algorithm2e}
\usepackage{booktabs} 
 \usepackage{epsfig}
 \usepackage{capt-of}
\allowdisplaybreaks

\newcommand{\defn}{\ensuremath{:  =}}

\newtheorem{theorem}{Theorem}[section]

\newtheorem{lemma}{Lemma}[section]

\newtheorem{remark}{Remark}[section]
\let\oldremark\remark
\renewcommand{\remark}{\oldremark\normalfont}

\renewenvironment{proof}{\vspace{-0.2cm} \noindent{\bf Proof. }}{\qed}
\newenvironment{proof-sketch}{\noindent{\bf Sketch of Proof}
  \hspace*{1em}}{\qed\smallskip\\}
\newenvironment{proof-idea}{\noindent{\bf Proof Idea}
  \hspace*{1em}}{\qed\smallskip\\}
\newenvironment{proof-of-lemma}[1][{}]{\noindent{\bf Proof of Lemma {#1}}
  \hspace*{1em}}{\qed\smallskip\\}
\newenvironment{proof-of-theorem}[1][{}]{\noindent{\bf Proof of Theorem {#1}}
  \hspace*{1em}}{\qed\smallskip\\}
\newenvironment{proof-of-proposition}[1][{}]{\noindent{\bf
    Proof of Proposition {#1}}
  \hsp\ace*{1em}}{\qed\smallskip\\}
\newenvironment{proof-attempt}{\noindent{\bf Proof Attempt}
  \hspace*{1em}}{\qed\bigskip\\}

\usepackage{amssymb}
\usepackage{pstricks}
\usepackage{graphicx}
\usepackage{tikz}
\usetikzlibrary{shapes,decorations}
\usepgflibrary{shapes.geometric}
\usepackage{amsmath}
\usepackage{subfigure}
 \usepackage{epsfig}
\usepackage{mdframed}

\allowdisplaybreaks
\long\def\symbolfootnote[#1]#2{\begingroup%
 \def\thefootnote{\fnsymbol{footnote}}\footnote[#1]{#2}\endgroup}

\renewcommand\eqref[1]{(\ref{#1})}

\usepackage{enumitem}
\setitemize{noitemsep,topsep=0pt,parsep=0pt,partopsep=0pt}

\long\def\symbolfootnote[#1]#2{\begingroup%
 \def\thefootnote{\fnsymbol{footnote}}\footnote[#1]{#2}\endgroup}

\newcommand{\CIT}{\mathsf{CIT}}

\newcommand\ind{\protect\mathpalette{\protect\independenT}{\perp}} \def\independenT#1#2{\mathrel{\rlap{$#1#2$}\mkern2mu{#1#2}}}

\newcommand{\Xf}{\mathfrak{X}}
\newcommand{\Pb}{\mathbb{P}}

\renewcommand\eqref[1]{(\ref{#1})}

\begin{document}

\twocolumn[

\aistatstitle{Active Learning for Undirected Graphical Model Selection}

\aistatsauthor{ Divyanshu Vats \And Robert D. Nowak \And Richard G. Baraniuk }

\aistatsaddress{ Rice University \And University of Wisconsin - Madison \And Rice University } ]

\begin{abstract}
This paper studies graphical model selection, i.e., the problem of estimating a graph of statistical relationships among a collection of random variables.  Conventional graphical model selection algorithms are passive, i.e., they require all the measurements to have been collected before processing begins.  We propose an active learning algorithm that uses junction tree representations to adapt future measurements based on the information gathered from prior measurements.  We prove that, under certain conditions, our active learning algorithm requires fewer scalar measurements than any passive algorithm to reliably estimate a graph.  A range of numerical results validate our theory and demonstrates the benefits of active learning.
\end{abstract}

\section{Introduction}
An important problem that arises in many applications is that of inferring the statistical relationships between a large collection of random variables.  For example, the random variable could represent expression values of a gene, opinions of a person, or stock returns of a company. Graphical models compactly represent statistical relationships using a graph.  The vertices in the graph represent random variables, and the edges in the graph represent statistical relationships between random variables \cite{WainwrightJordan2008}.  Although the graph may be of three types, namely directed, undirected, or mixed, we only study undirected graphs here.  Given measurements drawn from a graphical model, there are now several  algorithms for estimating the graph of statistical relationships.  See \cite{NicolaiPeter2006,BanerjeeGhaoui2008,AnimaTanWillsky2011b} for Gaussian graphical models, \cite{RavikumarWainwrightLafferty2010,
NetrapalliSanghaviAllerton2010,AnimaTanWillsky2011a} for discrete graphical models, and \cite{lafferty2012sparse} for nonparametric graphical models.

All conventional algorithms for learning graphical models are \textit{passive}, i.e., they rely on all the measurements being collected before any processing begins.  We envision several applications of {active learning for graphical models}, where future measurements are collected based on the information gathered from prior measurements and/or prior knowledge.  For example, in gene expression analysis,  once enough measurements have been obtained from a large collection of genes, subsequent measurements can be focused on a subset of genes with more complex interactions.  In social network analysis, measurements can be focused on a small subset of people rather than all people in the social network.  

Although there exists active learning algorithms for various statistical inference problems, including classification \cite{settlestr09}, sparse signal recovery \cite{haupt2011distilled}, clustering \cite{eriksson2011active}, multiple testing \cite{malloy2011limits}, matrix completion \cite{krishnamurthy2013sequential}, and causal structure discovery \cite{tong2001active},  the methods in these works do not necessarily apply to learning graphical models.  Furthermore, although there exists methods for designing optimal experiments for learning statistical models \cite{pukelsheim1993optimal}, we are not aware of any work that studies active learning for graphical models.

In this paper, we propose an active learning algorithm for learning the structure of the graph in a graphical model.  On a high level, our algorithm is summarized as follows.  Suppose we have a large graph that is composed of two or more subgraphs that may have complicated structures themselves, but have relatively few edges between them.  In principle it should be easier to identify the gross structure of the graph (i.e., the subsets of vertices corresponding to each subgraph and the few edges between these sets of vertices), then to identify the full graph structure.  So we pursue a sequential and active approach to learn the graph.  

First, we obtain full joint measurements of all the vertices and identify the gross structure.  The gross structure allows us to partition the large graph into multiple subgraphs.  We then identify the edges and the non-edges in each subgraph that have been estimated reliably.  Next, we collect additional, focused measurements, over a subset of the vertices to identify the edges that could not be reliably estimated using the past measurements.  The advantage of this sort of approach is that many of the measurements only involve a smaller subset of the vertices.  For this reason, the total number of \textit{scalar measurements} required for reliable graph estimation using this sort of active procedure can be significantly lower than the total number of scalar measurements required by conventional passive methods.

Theoretically, we establish sufficient conditions on the number of scalar measurements needed for reliable graph estimation using an active learning algorithm. Next, we analyze our algorithm when given additional knowledge about the absence of certain edges in the graph.  We prove that, under certain favorable conditions, an active learning algorithm can  estimate walk-summable Gaussian graphical models over $p$ vertices using only $O(p_{\min} \theta_{\min}^{-2} \log p_{\min})$ scalar measurements, while any passive algorithm necessarily requires $O(p \theta_{\min}^{-2} \log p_{\min})$ scalar meausurements.  Here, $p_{\min}$ is the size of the smallest cluster in a junction tree representation after incorporating the prior knowledge and $\theta_{\min}$ quantifies the intrinsic difficulty of the graphical model selection problem.  The particular conditions in our analysis depend on the positioning certain \textit{``weak edges"} in the graph and the scaling of the parameter $\theta_{\min}$.  Finally, we empirically demonstrate the benefits of our algorithm using numerical simulations.

%The rest of the paper is organized as follows.  Section~2 reviews undirected graphical models.  Section~3 presents our active learning algorithm.  Section~4 reviews an undirected graphical model selection algorithm that allows to closely study the theoretical properties of active learning.  Section~5 defines a family of graphical models for which active learning results in advantages over a passive learning algorithm.  Section~6 presents our main theoretical results.  Section~7 presents numerical results.  Section~8 concludes the paper.

\section{Undirected Graphical Models}

An undirected graphical model is a joint probability distribution, say $P_X$, defined on a graph $G^* = (V, E(G^*))$, where $V = \{1, . . . , p\}$ indexes the random vector $X = (X_1, . . . , X_p)$.  For any graph $G$, we use the notation $E(G)$ to denote its edges.  The vertices $V$ index the random variables and the edges $E(G^*)$ encode statistical relationships between the random variables.  In particular, when $P_X > 0$, undirected graphical models can be characterized using Markov properties.  One such Markov property is the \textit{global Markov property} which says that whenever a set of vertices $A$ and $B$ are separated by $S$, then $X_A$ is independent of $X_B$ given $X_S$.  Note that a set $S$ separates $A$ and $B$ if all paths from $A$ to $B$ pass through $S$.  In this paper, we consider the \textit{graphical model selection problem} of estimating the unknown graph $G^*$ given measurements drawn from the probability distribution $P_X$.  
%We use the notation $\Xf_{A}^n = (X_{A}^{(1)},\ldots,X_A^{(n)})$ $n$ i.i.d. measurements drawn from the probability distribution $P_{X_A}$, where $A$ is a subset of $V$.

\section{Active Learning Algorithm}
In this section, we present our active learning algorithm for graphical model selection.  Recall that our goal is to actively draw measurements from $P_X$.  Section~\ref{subsec:overview} discusses our algorithm.  Section~\ref{subsec:active} discusses a key step in our algorithm that determines the future measurements given prior measurements.

\subsection{Algorithm Overview}
\label{subsec:overview}

\input{figexampleactiveugms}

\begin{mdframed}
\textbf{Algorithm~1}: Active Learning
\begin{itemize}
\item Inputs: $A$, $\widehat{E}$, $\widehat{F}$, ${q}$, ${K}$, and ${\delta}$.
\item Initialization: $\Xf \gets \emptyset$
\item For $w = 1,2,\ldots,K$
\begin{itemize}
\item If $w = K$, then $\delta \gets 1$
\item $m \gets \lfloor \delta q / |A| \rfloor  $ ; $q \gets (1-\delta) q$
\item $\Xf_A^m \gets $ Draw $m$ i.i.d. samples from $P_{X_A}$. 
\item Update measurements: $\Xf \gets \Xf \cup \Xf^m_A$
\item Update $A$, $\widehat{E}$, and $\widehat{F}$ using Algorithm~2.
\end{itemize}
\item Estimate the remaining edges and combine with $\widehat{E}$ and $\widehat{F}$ to output $\widehat{G}$.
\end{itemize}
\end{mdframed}

Algorithm~1, which can be seen as an  extension of the active methods for sparse signal recovery \cite{haupt2011distilled,malloy2011limits} applied to graphs, presents our active learning algorithm for graphical model selection with the following inputs:
\begin{itemize}
\item Active vertices $A$:  We say that $A \subseteq V$ are \textit{active vertices} if all edges and non-edges over $A^c$ and those connecting $A^c$ to $A$ are \textit{known}.
\item Estimated edges $\widehat{E}$: Edges that have been estimated to be in the true graph.
\item Estimated non-edges $\widehat{F}$: Edges that have been estimated to \textit{not} be in the true graph.
\item Measurement budget $q$: Total number of scalar measurements Algorithm~1 should draw from $P_X$.
\item Number of measurement rounds $K$: Number of times Algorithm~1 draws measurements from $P_X$.
\item Fraction of measurements $\delta$: The fraction of scalar measurements drawn in each round.
\end{itemize}
The main idea in Algorithm~1 is to sequentially draw measurements from $P_X$ and check for edges and non-edges that can be reliably estimated using prior measurements.
Algorithm~1 initiates by drawing measurements from the active vertices $A$, where the number of measurements is determined by $q$ and $\delta$. Next, the sets $A$, $\widehat{E}$, and $\widehat{F}$ are updated using Algorithm~2, which is discussed in Section~\ref{subsec:active}.  In general, as illustrated in Figure~\ref{fig:exampleactiveugms}, as measurements are acquired, the size of the set $A$ decreases since parts of the graph are reliably estimated using prior measurements.

\subsection{Finding Active Vertices}
\label{subsec:active}
In this section, we discuss the challenging step in Algorithm~1 of updating the active vertices $A$, the edges $\widehat{E}$, and the non-edges $\widehat{F}$.  Our main idea  is to estimate two graphs, $H^+$ and $H^-$, such that $H^+$ is likely to contain all the true edges and $H^-$ is likely to contain a subset of the true edges.  The edges $\widehat{E}$ and $\widehat{F}$ can then be identified from $H^{-}$ and $H^{+}$, respectively.  We now want to devise an algorithm to find the active vertices $A$ given $H^{-}$ and $H^{+}$.
For a set $U$ and a graph $G$, let $G[U]$ be the \textit{induced subgraph} over $U$ that contains all edges from $G$ that only involve the vertices $U$.
Note the following:
\begin{itemize}
\item If $H^+ = H^-$, we clearly do not need any more measurements.
\item Suppose $U$ and $U'$ have the property that $U \backslash U'$ is separated from all other vertices.  If $H^+[U] = H^-[U]$, then we must have that $G^*[U] = H^+[U] = H^-[U]$.  In this case, there is no need to draw measurements from the vertices $U \backslash U'$ and the sets $\widehat{E}$ and $\widehat{F}$ can be modified accordingly.  We may still need to draw measurements from $U'$ since edges in other clusters may depend on $U'$.  
\item If $H^+[U] \ne H^-[U]$, all vertices over $U$ may need to be observed further.
\end{itemize}
To identify appropriate sets $U$, we use junction tree representations of the graph $H^+$.  Informally, a junction tree clusters vertices in a graph so that the resulting graph over the clusters is a tree; see \cite{Lauritzen1996} for more details.  In prior work, we have used junction trees to improve the performance of passive graphical model selection algorithms \cite{VatsNowakJMLR2014}.  As it turns out, since we are only interested in the clusters of the junction tree, it is sufficient to identify the cliques in a chordal graph of $H^+$; see \cite{GraphTheoryWest2000} for a definition of chordal graphs.  

A graph may have multiple chordal graphs.  An optimal chordal graph, which is computationally difficult to find, is defined so that the size of the maximum clique is the smallest.  Although finding optimal cliques is ``ideal" for our algorithm, it is not necessary for our algorithm to function properly.  In our implementation, we use linear time greedy heuristics \cite{BerryElimination2003}, which are known to output close to optimal chordal graphs \cite{JensenJenson1994}.  A summary of the above steps is shown in Algorithm~2.

To illustrate Algorithm~2, consider the graphs $G^*$, $H^+$, and $H^-$ in Figure~\ref{fig:h5}.  A simple calculation shows that the cliques of a chordal graph of $H^+$ are $\{1,2,4\}$, $\{2,3,4\}$, $\{3,4,5\}$, and $\{5,6\}$.  Comparing the induced subgraphs of $H^+$ and $H^-$ on the cliques, we identify that the edge $(5,6)$ is in the true graph.  Furthermore, if $(3,5) \notin H^+$, then the edge $(4,5)$ can also be identified to be in the true graph.

\begin{mdframed}
\textbf{Algorithm~2}: Find Active Vertices
\begin{itemize}
\item \textbf{Inputs:} $\Xf$, $\widehat{E}$ and $\widehat{F}$. 
\item  {Initialize:} $A \gets \emptyset$ 
\item  Estimate $H^+$ and $H^-$ (see Remark~\ref{sec:hh}).
\item  ${\cal V} \gets $ Cliques in chordal graph of $H^+$
\item For each clique $V_k \in {\cal V}$
\begin{itemize}
\item  If $H^{+}[V_k] \ne H^{-}[V_k]$, then $A \gets A \cup V_k$
\item  If $H^{+}[V_k] = H^{-}[V_k]$, then $\widehat{E} \gets$ Edges of $H^+[V_k]$, $\widehat{F} \gets$ Nonedges of $H^+[V_k]$
\end{itemize}
\item  \textbf{Return} $A, \widehat{E}$ and $\widehat{F}$
\end{itemize}
\end{mdframed}

\begin{figure}
\input{FindActiveFigure}
\caption{Illustration of Algorithm~2.}
\label{fig:h5}
\end{figure}

\begin{remark}
\label{sec:hh}
An important step in Algorithm~2 is computing the graphs $H^+$ and $H^-$.  Recall that we want $G^* \subseteq H^+$ and $H^- \subseteq G^*$.  In our numerical simulations, we use stability selection \cite{meinshausen2010stability}, with appropriate thresholds, to select $H^+$ and $H^-$.  We refer to Appendix~\ref{app:a} for more details.
\end{remark}

\begin{remark}
Both Algorithms~1 and 2 are independent of the choice of the graphical model selection algorithm.  Furthermore, the computational complexity of the active learning algorithm is dominated by the computation of $H^+$ and $H^-$.  Thus, the overall complexity is roughly $O (K {\cal I} )$, where $O({\cal I})$ is the complexity of graphical model selection.  As will be clear from the theoretical results and the numerical simulations, the additional computations required for active learning is a small price to pay for the potential benefits of using active learning for improved graph estimation.
\end{remark}

%Further, we use the   Algorithm~\ref{alg:activenodes} is explained as follows:
%\setlist[itemize]{leftmargin=0.4cm,itemsep=0.05cm,topsep=0.05cm,parsep=0.05cm}
%\begin{itemize}[label=$\bullet$]
%\item The first step (Line~2) is to compute graphs $H^+$ and $H^-$ using Algorithm~\ref{alg:huhl} such that $G^* \subseteq H^+$ and $H^- \subseteq G^*$.  In other words, $H^+$ is a superset of $G^*$ and $H^-$ is a subset of $G^*$.  A discussion about Algorithm~\ref{alg:huhl} is presented later.  
%\item Next (Line~3), we find all cliques in a chordal graph of $H^+$. As it turns out, these cliques correspond to the clusters in a junction tree representation of $H^+$.
%These cliques can be easily computed in linear time.
%\item We now want to examine the cliques $\{V_k\}$ to determine the set of active vertices.  We say a vertex is \textit{active} if future measurements are required from that vertex to estimate edges in the graph.
%
%We now use the following key observation: if there exists even one edge in $H^+[V_k]$ that is not in $G^*$, then all vertices in $V_k$ are necessarily active.  On the other hand, if all edges and non-edges in $H^+[V_k]$ and $G^*[V_k]$ are equal
%Since $H^-$ is a subset of $G^*$, the condition above is clearly satisfied when $H^+[V_k] \ne H^-[V_k]$.  We check for this condition in Line~6 of Algorithm~\ref{alg:activenodes}.  
%\item If the condition in Line~6 is not satisfied, which means that $H^+[V_k] = H^-[V_k]$, then it must be the case that $H^+[V_k] = H^-[V_k] =  G^*[V_k]$.  We update the set of edges $\widehat{E}$ and $\widehat{F}$ accordingly.
%\end{itemize}

\section{Conditional Independence Testing}
\begin{mdframed}
\textbf{Algorithm~3.} $\CIT$($\Xf^n_V$, $\kappa$, $\eta$): Conditional independence testing for graphical model selection
\begin{itemize}
\item \textbf{Inputs:} $\Xf^n_V$: $n$ i.i.d. measurements; $\kappa$: An integer that controls the computational complexity; $\tau_n$: threshold that controls the sparsity of graph.
\item $\widehat{G} \gets $ Complete graph over $p$ vertices.
\item \textbf{for} each $(i,j) \in E(\widehat{G})$
\begin{itemize}
\item {If} $\exists$ $S$, $|S| \le \kappa$, s.t. $| \widehat{\rho}_{ij|S} | \le \tau_n$, then delete edge $(i,j)$ from $\widehat{G}$.
\end{itemize}
\item \textbf{Return} $\widehat{G}$.
\end{itemize}
\end{mdframed}
In this section, we review a graphical model selection algorithm to study the advantages of our active learning algorithm.  In particular, we review Algorithm~3, called $\CIT$, which uses conditional independence tests to estimate a graph.  This method is not new, and goes back to the SGS-Algorithm \cite{SGSAlgorithm} for learning Bayesian networks.  The conditional independence test used in $\CIT$ is to threshold the empirical conditional correlation coefficient (see (A3) for definition).  Recently, \cite{AnimaTanWillsky2011a,AnimaTanWillsky2011b} studied the regimes under which a conditional independence test based graphical model selection algorithm has attractive sample complexity.  Although the computational complexity of Algorithm~3 is $O(p^{\kappa+2})$, where $\kappa$ is an input to the algorithm, the PC-Algorithm \cite{PCAlgorithm} can be used to significantly speed up the computations.

To characterize the performance of Algorithm~3, we consider the following assumptions.
\setlist[enumerate]{leftmargin=0.9cm,itemsep=0.05cm,topsep=0.05cm,parsep=0.05cm}
\begin{enumerate}[label=({A}\arabic*)]
\item $P_X$ is a multivariate normal distribution with mean zero and covariance $\Sigma$ such that $\max_{i,i} \Sigma_{i,i} \le M < \infty$, where $M$ is a constant.
\item $X_i \ind X_j | X_S$ $\Longleftrightarrow$ $i$ and $j$ are separated by $S$.
\item $\sup |\rho_{ij|S}| \!<\! 1$, where $\rho_{ij|S} \!= \!\frac{\Sigma_{ij|S}}{\sqrt{\Sigma_{i,i|S} \Sigma_{j,j|S}}}$ and $\Sigma_{i,i|S}\! = \!\Sigma_{i,j} - \!\Sigma_{i,S} \Sigma_{S,S}^{-1} \Sigma_{S,j}$.  Note that $\widehat{\rho}_{ij|S}$ is computed using the empirical covariance matrix.
\item If $(i,j) \notin E(G^*)$, there exists a \textit{minimal separator} of size $\eta$ that separates $i$ and $j$.
\end{enumerate}

The Gaussian assumption in (A1) is for simplicity.  We can use the results in \cite{AnimaTanWillsky2011a} to generalize the analysis to discrete distributions.  Assumption~(A2) is sometimes called the faithfulness condition.  The parameter $\rho_{ij|S}$ in (A3) is the conditional correlation coefficient.  Whenever $(i,j) \notin G^*$, then $\rho_{ij|S} = 0$.  Moreover, using (A2), we have that $\rho_{ij|S} = 0$ if and only if $(i,j) \notin G^*$.  This justifies the use of the empirical conditional correlation coefficient, $\widehat{\rho}_{ij|S}$, to test for conditional independence in Algorithm~1.  The \textit{minimal separator} in (A4) is defined as a separator $S$ for $(i,j) \notin E(G^*)$ such that no proper subset of $S$ separates $i$ and $j$.  The parameter $\eta$ in (A4) implicitly places limits on the sparsity of the graph.  For example, we can easily upper bound $\eta$ by the maximum degree of the graph.  However, for many graphs, this upper bound is very loose.  For example, $\eta = 1$ for trees, but the maximum degree can be as large as $p-1$.  Finally, we define the minimal conditional correlation coefficient as follows:
\begin{align}
\rho_{min} \defn \min_{ (i,j) \in G^*, |S| \le \eta}  \; |\rho_{ij|S}| \,. \label{eq:rhomin}
\end{align}
Now, suppose we are given $n$ i.i.d measurements $\Xf^n_V = (X^{(1)}_V,\ldots,X^{(n)}_V)$ drawn from $P_X$.  We work within a high-dimensional framework so that the various problem parameters can scale arbitrarily as $p \rightarrow \infty$.
%
%\begin{align}
%\tau_n > \frac{c_1 (\eta+2) \log p}{n-\eta} \,,\quad \tau_n = 0.9 \rho_{\min} \,,
%\label{eq:taun}
%\end{align}
%where $c_1>0$ is a constant.  
%
We have the following theorem.
\begin{theorem}
\label{thm:firstthm}
Suppose Assumptions (A1)-(A4) hold and let $\widehat{G} = \CIT(\Xf_V^n,\eta,\tau_p)$, where $\tau_p = 0.9 \rho_{\min}$.  For constants $c_1,c_2 > 0$, if $\rho_{\min} > \frac{c_1 (\eta+2) \log p}{(n-\eta)}$ and
\[
n \ge \eta + c_2 \rho_{\min}^{-2} (\eta+2) \log(p)\,,
\] 
then $\Pb(\widehat{G} = G^*) \rightarrow 1$ as $p \rightarrow \infty$.
\end{theorem}
The proof of Theorem~\ref{thm:firstthm}, outlined in Appendix~\ref{app:b}, is based on methods in \cite{KalischBuhlmann2007} and \cite{AnimaTanWillsky2011b}.  We note that Theorem~\ref{thm:firstthm} differs slightly from the results on conditional covariance based testing in \cite{AnimaTanWillsky2011b}.  In particular, the result in \cite{AnimaTanWillsky2011b} is based on local separators, while the result in Theorem~\ref{thm:firstthm} is based on exact separators between non-edges of the graph.  In general, the size of a local separator is less than the size of the exact separator.  
Although we use Theorem~\ref{thm:firstthm} to analyze our active learning algorithms, our analysis can be easily derived using the results from \cite{AnimaTanWillsky2011b}.

%We note that although the results in Theorem~\ref{thm:firstthm} are only sufficient for accurate graph recovery, the bounds on the number of observations nearly match the necessary conditions for accurate graph recovery outlined in \cite{wang2010information}.  In the next Section, we theoretically study the benefits of active learning.

\section{ A Graph Family with Two Clusters}
\label{sec:graphfamily}

In this section, we define a family graphical model family to highlight the advantages of active learning.  In the definitions that follow, operations over graphs correspond to operations over the vertices and edges. 

\noindent
${\cal G}_{p,p_1,p_2,\eta,d} \defn$ Family of graphs over $p$ vertices such that $G = G_1 \cup G_2$, where $G_1$ and $G_2$ are characterized as follows.  Arbitrarily select two sets of vertices $V_1$ and $V_2$, such that $V_1 \cup V_2 = V$ and $T = V_1 \cap V_2$, where $|T| \le 1$.  Let ${\cal G}_{p,\eta,d}$ be the set of all graphs over $p$ vertices with maximum degree $d$ and minimal separator of size less than or equal to $\eta$.  Assume $G_k = (V_k,E(G_k)) \in {\cal G}_{p_k,\eta,d}$, for $k = 1,2$, where $p_k = |V_k|$.  Note that, since $|T| \le 1$, $G[V_1] = G_1$ and $G[V_2] = G_2$.

\smallskip

\noindent
$\text{$\Theta(G)$}\defn $ Inverse covariance matrix of a zero mean Gaussian graphical model  on a graph $G$.

\smallskip

\noindent
${\cal G}_{p,p_1,p_2,\eta,d}(\theta_1,\theta_2) \defn$ Set of all possible inverse covariance matrices $\Theta(G)$, where $G \in {\cal G}_{p,p_1,p_2,\eta,d} $, such that
$\displaystyle{\min_{(i,j) \in E(G[V_k]) } \frac{|\Theta_{ij}(G)|}{\sqrt{\Theta_{ii}(G) \Theta_{jj}(G)}} \!\ge\! \theta_k}$, for $k = 1,2$, where $\theta_1$ and $\theta_2$ quantify the minimal conditional covariances over $V_1$ and $V_2$ given all other variables.

\smallskip

\begin{figure}
\begin{center}
\input{FirstFigExample}
\end{center}
\caption{(a)-(b) Examples of the graphs in Section 5. (c) Measurement scheme in Algorithms~4 and 5.}
\label{fig:firsteg}
\end{figure}

Throughout this paper, we assume that $G^* \in {\cal G}_{p,p_1,p_2,\eta,d}$ and that the Gaussian graphical model has zero mean with inverse covariance $\Theta(G^*) \in  {\cal G}_{p,p_1,p_2,\eta,d}(\theta_1,\theta_2)$.  From the definition, it is clear that $G^*$ admits a two cluster decomposition, as in Figure~\ref{fig:firsteg}(a)-(b), where there exists a set of vertices $T$ that separates the vertices $V_1 \backslash T$ and $V_2 \backslash T$.  In words, this means that all paths from a vertex in $V_1 \backslash T$ to a vertex in $V_2 \backslash T$ pass through $T$.  When $T = \emptyset$, there are no edges between $V_1$ and $V_2$.  Note that the assumption $|T| \le 1$ is only enforced to simplify our analysis; see Remark~\ref{rem:r} for more details.

Next, we define three parameters on $\Theta(G^*)$ that will be important in the analysis of our algorithm:
\begin{align}
\rho_0 &\defn \min_{(i,j) \in E(G^*), |S| \le |T|} |\rho_{ij|S}| \label{eq:rho} \,,\\
\rho_1 &\defn \min_{(i,j) \in E(G^*[V_1]), |S| \le \eta, S \subset V_1} |\rho_{ij|S}| \,, \label{eq:rho1}\\
\rho_2 &\defn \min_{(i,j) \in E(G^*[V_2]), |S| \le \eta, S \subset V_2} |\rho_{ij|S}| \,. \label{eq:rho2}
\end{align}
Informally, $\rho_0$ quantifies the difficulty in learning the two cluster decomposition, $\rho_1$ quantifies the difficulty in learning the edges over $V_1$, and $\rho_2$ quantifies the difficulty in learning the edges over $V_2$.

Finally, we use the results in \cite{AnimaTanWillsky2011b} to  relate the parameters $\rho_1$ and $\rho_2$ to $\theta_1$ and $\theta_2$, respectively.  In what follows, the various parameters defined on the graphical model are assumed to scale with $p$ and we use the following notations: $\|M\|$ is the spectral norm of a matrix and $f_p = \Omega(g_p)$ means that for sufficiently large $p$, there exists a constant $c$ such that $f_p \ge c g_p$.

\begin{theorem}[$\!\!\!$\cite{AnimaTanWillsky2011b}]
\label{thm:rhotheta}
Let $\Theta(G^*) \in {\cal G}_{p,p_1,p_2,\eta,d}(\theta_1,\theta_2)$.
Suppose $\Theta_{ii}(G^*) = 1$ $\forall$ $i$ and $\Theta_{ij}(G^*) \le 0$ for $i \ne j$.  If $\| I - |\Theta(G^*)|\| = \alpha < 1$, where $\alpha$ is a constant, then $\rho_1 = \Omega(\theta_1)$ and $\rho_2 = \Omega(\theta_2)$.
\end{theorem}

Theorem~\ref{thm:rhotheta} shows that $\rho_1$ and $\rho_2$ are asymptotically lower bounded by $c\theta_1$ and $c\theta_2$, respectively, where $c$ is an appropriate constant.  The condition on $\Theta(G^*)$, although restrictive, can be generalized so that $\Theta(G^*)$ is a walk-summable graphical model \cite{malioutov2006walk}.  For simplicity, we avoid stating the conditions and refer to Lemma~14 in \cite{AnimaTanWillsky2011b} for more technical details.

%\begin{remark}
%\label{rem:vv1}
%The assumption that $|T| \le 1$ is only placed to simplify the analysis that we present in subsequent sections.  In practice, $|T| \le \eta$, and our analysis can be modified accordingly at the cost of some additional technicalities.
%\end{remark}

\section{Theoretical Analysis of a Two-Stage Active Learning Algorithm}
In this section, we derive necessary and sufficient conditions on the number of \textit{scalar measurements} required for reliable estimation of the unknown graph using our active learning algorithm.  Recall that if we draw $n$ measurements from $p$ vertices, then the number of scalar measurements is $np$.

Section~\ref{subsec:twostage} presents sufficient conditions for a modified version of Algorithm~1 that is designed for graphs in the two-cluster graph family defined in Section~\ref{sec:graphfamily}.  Section~\ref{subsec:prior} presents sufficient conditions when given prior knowledge about the absence of certain edges in the graph.  Section~\ref{subsec:nec} compares the sufficient conditions to necessary conditions required by any passive graphical model selection algorithm.

\subsection{Sufficient Conditions}
\label{subsec:twostage}
Recall that Algorithm~1 uses Algorithm~2 to update the set of active vertices.  Unfortunately, an analysis of Algorithm~2 is not within the scope of this paper and is left for future work.  Instead, we replace Algorithm~2 with another method, specific to the two-cluster decomposition.  The details of the active learning algorithm we analyze is given in Algorithm~4.
\begin{mdframed}
\noindent
\textbf{Algorithm 4.} Two-Stage Active Learning
\setlist[enumerate]
{leftmargin=0.5cm,itemsep=0.05cm,topsep=0.05cm,parsep=0.05cm}
\begin{enumerate}
\item[1)] Draw $n_0 = \eta + c_2 \log p \max\left\{3\rho_0^{-2}, \rho_2^{-2}(\eta+2) \right\} $ measurements, $\Xf_V^{n_0}$, from $V$.
\item[2)] $\widehat{G} \!\gets\! \CIT(\Xf_V^{n_0},\eta,\tau_0)$, where $\tau_0 \!\!= \! \!0.9 \min\{\rho_0,\rho_2\}$. % and $\tau > c_1(\eta+2)\log p/ (n_0-\eta)$.  
\item[3)] Find $\widehat{V}_1$, $\widehat{V}_2$, and $\widehat{T}$ such that $\widehat{T}$ separates $\widehat{V}_1 \backslash \widehat{T}$ and $\widehat{V}_2 \backslash \widehat{T}$ in $\widehat{G}$, $|\widehat{T}| = 1$, and $\widehat{G}[V_2] = G^*[V_2]$.
\item[4)] Let $\widehat{p}_1 = |\widehat{V}_1|$.  Draw $n_1 = \eta + c_2 \rho_1^{-2} (\eta+2) \log \widehat{p}_1 - n_0$ measurements from $\widehat{V}_1$.
\item[5)] $\widehat{G}_1 \gets \CIT(\Xf_{\widehat{V}_1}^{n_0+n_1},\eta,\tau_1)$, where $\tau_1 = 0.9 \rho_1$. % > c_1(\eta+2)\log p_1/ (n_0+n_1-\eta)$.  
\item[6)] Return $\widehat{G} = \widehat{G}_1 \cup \widehat{G}[\widehat{V}_2]$.
\end{enumerate}
\end{mdframed}
Algorithm~4 corresponds to Algorithm~1 with two rounds of measurements ($K = 2$), $A = V$, and $\widehat{E} = \widehat{F} = \emptyset$.  The crux of Algorithm~4 is illustrated in Figure~\ref{fig:firsteg}(c), where we first draw measurements from all the vertices and then focus the next round of measurements over $\widehat{V}_1$.  We \textit{do not} draw measurements over $\widehat{V}_2$ since the edges and the non-edges over $\widehat{V}_2$ are estimated using the first round of measurements.

Before presenting our result regarding Algorithm~4, we state three additional assumptions that we impose on the graphical model.
\begin{enumerate}
\item[(A5)] $\rho_1^{-2} (\eta+2) \log p_1 > \max\{3\rho_0^{-2},\rho_2^{-2} (\eta+2) \} \log p$
\item[(A6)] $0.9 \rho_1 > c_1(\eta+2)\log p_1/ (n_0+n_1-\eta)$
\item[(A7)] $0.9 \min\{\rho_0,\rho_2\} > c_1(\eta+2)\log p/ (n_0-\eta)$
\end{enumerate}
Informally, (A5) ensures that the subgraph over $V_1$ has a more complex structure and requires more measurements to reliably estimate all the edges over $V_1$.  Both (A6) and (A7) ensure that the parameters $\rho_0$, $\rho_1$, and $\rho_2$ are not too small so that the true edges can be distinguished from the non-edges.
\begin{theorem}
\label{thm:firstthmA}
Under Assumptions (A1)-(A7), Algorithm~4 outputs the true graph with probability converging to one as $p \rightarrow \infty$.  Furthermore, for constants $c_1,c_2 > 0$, the number of scalar measurements drawn by Algorithm~4 is equal to

$(p-p_1) c_2 \max\left\{3\rho_0^{-2}, \rho_2^{-2}(\eta+2) \right\} \log p $ 

$\hspace{3.7cm}+ p\eta +  p_1 c_2 \rho_1^{-2} (\eta+2) \log p_1 \,.$
%\begin{align*}
%(p-p_1) c_2 \max\left\{3\rho_0^{-2}, \rho_2^{-2}(\eta+2) \right\} \log p \nonumber  \\ 
%+ p\eta +  p_1 c_2 \rho_1^{-2} (\eta+2) \log p_1 \,.
%\end{align*}
\end{theorem}
The proof of Theorem~\ref{thm:firstthmA}, outlined in Appendix~\ref{app:c}, first uses Theorem~\ref{thm:firstthm} to show that $n_0$ measurements are sufficient to estimate the two cluster decomposition and the edges over $V_2$, and then again uses   Theorem~\ref{thm:firstthm} to show that $n_0+n_1$ measurements are sufficient to estimate the edges over $V_1$.  Note that Algorithm~4 does not necessarily identify the clusters $V_1$ and $V_2$ in step~3.  However, as shown in the proof of Theorem~\ref{thm:firstthmA}, given $n_0$ measurements, $\widehat{V}_1 \subseteq V_1$ and $V_2 \subseteq \widehat{V}_2$ with high probability.  We now make some additional remarks.

\begin{remark}
We emphasize that Algorithm~4 does \textit{not} assume that $V_1$ and $V_2$ are known.  Instead, Algorithm~4 only assumes that the parameters 
$\rho_0$, $\rho_1$, and $\rho_2$ are known.  Given these parameters, step~3 of Algorithm~4, where we check if $\widehat{G}[V_2] = G^*[V_2]$, can be implemented using the bounds for the $\CIT$ algorithm in Theorem~\ref{thm:firstthm}.  Furthermore, if $G^*$ does not admit a two-cluster decomposition, then $\widehat{V}_1 = V$ and $\widehat{V}_2 = \emptyset$, in which case Algorithm~4 will mirror the passive $\CIT$ algorithm.
\end{remark}

%Although Algorithm~2 is \textit{not practical}, since the choice of $n_0$ and $n_1$ depend on the unknown parameters of the graph, the analysis in Theorem~\ref{thm:firstthmA} allows us to understand the limits of an active learning algorithm.

\begin{remark}
\label{rem:r}
Recall from the definition of $G^*$ in Section~\ref{sec:graphfamily} that we imposed the simplistic assumption that $|T| \le 1$.  At the cost of some additional technicalities, Theorem~\ref{thm:firstthmA} can be extended to the case when $|T| > 1$.  The main change in the analysis will be to consider a slightly larger set $V_1$ to ensure that the edges over $T$ can be accurately estimated.
\end{remark}

%\begin{remark}
%\label{rem:vv1}
%The assumption that $|T| \le 1$ is only placed to simplify the analysis that we present in subsequent sections.  In practice, $|T| \le \eta$, and our analysis can be modified accordingly at the cost of some additional technicalities.
%\end{remark}

%An important step in Algorithm~2 is Step~3 that identifies the remaining edges that need to be estimated.  This step entails ensuring that the edges removed from the graph are necessarily not in $G^*$ and identifying, among the estimated edges, the edges that belong to $G^*$.  We assume $\widehat{V}_1 = V_1$ to highlight the optimal performance of Algorithm~2.  As we shall see later, this step can be implemented in practice using Algorithm~5 in Section~\ref{sec:prac}.

\begin{remark}
The choice of $n_0$ and $n_1$, and the subsequent analysis, is assuming the two cluster decomposition.  In practice, the graph $G^*$ can admit multiple two cluster decompositions.  Subsequently, Algorithm~4 can be tailored for such decompositions.  Thus, we can derive multiple bounds for the scalar measurements required for Algorithm~4 and the optimal one will correspond to the minimum over all two cluster decompositions of the graph $G^*$.
\end{remark}

\begin{remark}
It is easy to see that if (A5) holds, then the difference between the the scalar measurements required for Algorithm~4 and the scalar measurements required for the passive $\CIT$ algorithm is $O((p-p_1) \rho_1^{-2} \log p_1)$.  This suggests that when $p_1 \ll p$, the advantages of using Algorithm~2 may be much more pronounced.  Unfortunately, it is not clear if this analysis is tight since we are comparing the differences between two sufficient conditions.  Regardless, our numerical simulations in Section~6 clearly show the benefits of active learning.
\end{remark}

\subsection{Using Prior Knowledge}
\label{subsec:prior}

In this section, we analyze a variant of Algorithm~4 when given a priori knowledge that there exists no edges between $V_1 \backslash T$ and $V_2 \backslash T$ in $G^*$.  This information could be extracted from prior knowledge about the graphical model of interest.  For example, when studying financial data from companies, there may be prior knowledge available about the sectors of different companies.  When studying gene expression data, there may be prior knowledge available about the different pathways genes belong to.  We show that using such prior knowledge to adapt measurements can lead to significant reductions in the sample complexity of learning the true graph.

In Algorithm~5, we modify Algorithm~4 to take into account the prior knowledge about the graph.

\begin{mdframed}
\textbf{Algorithm 5.}
Given that $T$ separates $V_1 \backslash T$ and $V_2 \backslash T$, implement Algorithm~4 with Steps~1 and 2 replaced by
\begin{itemize}
\item[1)] Draw $n_0$ measurements, $\Xf_V^{n_0}$, from $V$ such that $n_0= \eta + c_2 \rho_2^{-2}  (\eta+2)\log(p_2)$.
\item[2)] $\widehat{G} \gets \CIT(\Xf_V^{n_0},\eta,\tau_0)$, where $\tau_0 = 0.9 \rho_2$. % and $\tau > c_1(\eta+2)\log p_2/ (n_0-\eta)$.  
\end{itemize}
\end{mdframed}
Algorithm~5 simply changes the initial measurements in Algorithm~4 to account for the fact that some non-edges in $G^*$ are already known.  In Algorithm~1, this corresponds to appropriately specifying the set $\widehat{F}$.  Before stating the main result regarding Algorithm~5, which follows easily from Theorem~\ref{thm:firstthmA}, we consider the following assumptions that are analogous to (A5)-(A7).
\begin{enumerate}
\item[(A5$'$)] $\rho_1^{-2} \log p_1 > \rho_2^{-2} \log p$
\item[(A6$'$)] $0.9 \rho_2 > c_1(\eta+2)\log p_2 / (n_0-\eta)$
\item[(A7$'$)] $0.9 \rho_1 > c_1(\eta+2)\log p_1/ (n_0+n_1-\eta)$
\end{enumerate}
Note that $n_0$ and $n_1$ in (A6$'$)-(A7$'$) are defined in Algorithm~5.
\begin{theorem}
\label{thm:firstthmB}
Under Assumptions~(A1)-(A4), (A5$\,'$)-(A7$\,'$) and given that $T$ separates $V_1 \backslash T$ and $V_2 \backslash T$, Algorithm~5 outputs the true graph with probability converging to one as $p \rightarrow \infty$.  Furthermore, for constants $c_1,c_2 > 0$, the number of scalar measurements is equal to

$(p\!-\!p_1) c_2 \rho_2^{-2}(\eta+2) \log p_2 \!+\! p \eta +  p_1 c_2 \rho_1^{-2} (\eta\!+\!2) \log p_1.$
\end{theorem}
The only difference between Theorem~\ref{thm:firstthmA} and Theorem~\ref{thm:firstthmB} is that the scalar measurements in the later theorem no longer depends on $\rho_0$, the parameter that quantifies the difficulty in learning the two cluster decomposition.
An alternative active learning method is to draw measurements from $V_1$ and $V_2$ separately.  As long as $|T|$ is small, this strategy will roughly need the same number of scalar measurements as Algorithm~3.  However, if there are constraints on the number of joint measurements a system can make, then this later strategy could be more useful.  For example, if the measurements are acquired from a sensor network, then there may be limits on the number of joint measurements sensors can transmit so as to conserve the battery life of sensors.

\subsection{Comparison to Necessary Conditions}
\label{subsec:nec}

We now compare the sufficient conditions in Theorem~\ref{thm:firstthmB} to the necessary conditions for \textit{any} passive algorithm.  Let $\Xf_{V}^n$ be $n$ i.i.d.\ samples drawn from ${\cal N}(0,\Theta^{-1}(G^*))$, where $\Theta(G^*) \in {\cal G}_{p,p_1,p_2,\eta,d}(\theta_1,\theta_2)$.  Let $\psi$ be a graph decoder that takes as input $\Xf_V^n$ and outputs a graph in ${\cal G}_{p,p_1,p_2,\eta,d}(\theta_1,\theta_2)$.  For any decoder $\psi$, define the maximal probability of error as
\[
p_e(\psi) = \max_{\Theta(G) \in {\cal G}_{p_1,p_2,\eta,d}(\theta_1,\theta_2)} \Pb (\psi(\Xf_V^n) \ne G) \,,
\]
where the probability is with respect to the product distribution of $({\cal N}(0,\Theta^{-1}(G)))^n$ over $n$ i.i.d. observations.  We say a graph decoder is high-dimensional consistent if $p_e(\psi) \rightarrow 0$ as $p \rightarrow \infty$.
\begin{theorem}
\label{thm:lb}
Suppose $\theta_1,\theta_2 \in [0,0.5]$ and the decoder $\psi$ is given prior knowledge that there are no edges between $V_1 \backslash T$ and $V_2 \backslash T$.  A necessary condition for high-dimensional consistent graphical model selection over a Gaussian graphical model with the inverse covariance matrix in the set ${\cal G}_{p,p_1,p_2,\eta,d}(\theta_1,\theta_2)$ is 
\begin{align*}
n > \frac{1}{2}\max \left\{\theta_1^{-2}  \log\frac{p_1-d-1}{2e}, \theta_2^{-2}\log\frac{p_2-d-1}{2e} \right\} \,.
\end{align*}
\end{theorem}
The proof of Theorem~\ref{thm:lb}, given in the supplement, uses information-theoretic methods from \cite{WangWainwright2010}.   

We now compare the passive and active algorithms.  Suppose $\theta_1$ is small enough so that 
$\theta_1^{-2} \log(p_1-d-1) > \theta_2^{-2} \log(p_2-d-1)$, and $\theta_1^{-2} \log(p_1) > \theta_2^{-2} \log(p_2)$.  Then, the necessary conditions on the number of scalar measurements for consistent selection by \textit{any} passive algorithm scales as
\begin{align}
q_{passive} = \Omega( p \theta_1^{-2} \log (p_1-d-1)).
\end{align}
On the other hand, using Theorem~\ref{thm:rhotheta} in Theorem~\ref{thm:firstthmB}, and assuming that $\eta$ is a constant, the sufficient conditions for the active method in Algorithm~5 scales as
\begin{align}
q_{active} = \Omega( (p-p_1) \theta_2^{-2} \log p_2 + p_1 \theta_1^{-2} \log p_1) \,.
\end{align}
Now, consider the condition
\begin{align}
%&p_1 \theta_1^{-2} \log p_1 > (p-p_1) \theta_2^{-1} \log p_2 \\
\theta_1^2 < \theta_2^{2} \frac{p_1 \log p_1}{(p-p_1) \log p_2} \,.
\label{eq:13}
\end{align}
A simple calculation shows that if (\ref{eq:13}) holds, then $q_{active} = \Omega(p_1 \theta_1^{-2} \log p_1)$.  Thus, if $d \ll p_1 \ll p$, and (\ref{eq:13}) holds, then Algorithm~5 requires far fewer number of scalar measurements than any other passive algorithm.  To get an understanding of the condition in (\ref{eq:13}), suppose $p_1 = \sqrt{p}$ and $|T| = 0$.  Then, $\theta_1^2 = O( \theta_2^2/\sqrt{p})$.  In other words, the advantages of active learning are substantial when $\theta_1$ is much smaller than $\theta_2$.

Finally, we note that our analysis is only for Algorithm~5, where information about the graph decomposition is known to the algorithm.  An open problem is to study how the performance of Algorithm~4 compares to the necessary conditions when no information about the graph decomposition is given.

%\section{From Theory to Practice}
%\label{sec:prac}
%
%So far, we have analyzed the theoretical performance of the active learning methods in Algorithms~2 and 3.  However, as mentioned in Remark~\ref{rem:np}, our algorithms are not practical since the steps depend on unknown parameters of the graphical model.  In this Section, we illustrate how the active learning steps can be implemented in practice.  Our main strategy for active learning, which is illustrated in Figure~\ref{fig:exampleactiveugms}, is as follows:
%\begin{quote}
%\emph{Draw measurements from a set of active vertices $A$ such that all edges and non-edges over $A^c$ and those connecting $A^c$ to $A$ have been estimated correctly using either prior measurements or prior knowledge.}
%\end{quote}
%Informally, the \textit{active vertices} are used to represent the set of edges that have yet to be estimated in the graph.  For example, in Algorithms~2 and 3, the initial active vertices are all the $p$ vertices, but once all edges over $V_2$ have been estimated, the active vertices become $V_1$.

\section{Numerical Results}

In this section, we present numerical results that highlight the advantages of our active learning algorithm.  For all synthetic results, we assume that $P_X$ is multivariate Gaussian with mean zero and covariance $\Sigma$.   Define the inverse covariance matrix by $\Theta = \Sigma^{-1}$.  If $P_X$ is Markov on $G^*$, it is well known that $(i,j) \notin G^*$ implies that $\Theta_{ij} = 0$.

We consider three different kinds of synthetic graphical models and assume that $\Theta_{ii} = 1$ for $i = 1,2,\ldots,p$.  For all graphs considered below, we assume that the first $p_1$ vertices are \textit{weak edges} so that the absolute value of the non-zero entries over these vertices is smaller than the other non-zero entries.  We refer to Appendix~\ref{app:e} for results on scale-free graphs.

\noindent
\textbf{Chain Graph:} $\Theta_{i,i+1} = \rho_1$ for $i = 1,\ldots,p_1$ (weak edges) and $\Theta_{i,i+1} = \rho_2$ for $i = p_1+1,\ldots,p$ (strong edges).  Let $\rho_1 = 0.1$, $\rho_2 = 0.3$, and $\Theta_{ij} = \Theta_{ji}$.

\noindent
\textbf{Hub Graph:}  The first $p_1$ vertices are partitioned into vertices of size $10$ and the remaining vertices are partitioned into vertices of size $5$.  For each partition, all vertices are connected to one vertex.  $\Theta_{ij}$ is constructed so that $\Theta_{ij} = 1/d_{ij} - \epsilon$, where $d_{ij}$ is either the degree of vertex $i$ or the degree of vertex $j$, depending on which one is larger.  The scalar $\epsilon = 10^{-4}$.  The above construction ensures that the matrix $\Theta$ is positive and symmetric.

\noindent
\textbf{Cluster Graph:}  All vertices are partitioned into clusters of size $20$.  For the first $p_1/20$ clusters, the edges over each cluster are generated using an Erdos-Renyi (ER) random graph model so that the probability that each edge appears in a cluster is $0.2$.  For the remaining clusters, the graph over each cluster is an ER graph with edges appearing with probability $0.1$.  The inverse covariance matrix is constructed as in the Hub graph case.  This construction ensures that the edges corresponding to the first $p_1 = 0.2p$ vertices have lower partial correlation values than all other edges.

%\smallskip

%\noindent
%\textbf{Chain+Edges Graph:}  We start with a chain graph and then add additional edges using the ER model.  We add edges to the first $p_1$ vertices using an ER model with probability $\log(p_1)/p_1$.  We add edges to the remaining $p_1+1,\ldots,p$ vertices with probability $\log(p-p_1)/(10(p-p_1))$.  The inverse covariance is constructed as in the Hub graph case.
%\smallskip

\begin{table*}
\centering
\caption{Cluster graph with $p = 400$ vertices}
\label{tab:clusterp400}
{\small{
\input{tabclusterp400}}}
\end{table*}

\subsection{Methodology}
We use the $\CIT$ with $\kappa = 1$ to perform the active learning component (computing $H^+$ and $H^-$) of our algorithm and $\CIT$ with $\kappa = 2$ to estimate the final graph.  In all experiments, $K = 5$ and $\delta = 0.5$.

%Motivated from the theoretical results in Section~2, we use a fixed threshold of $\tau = \sqrt{2\log(p)/(n-1)}$ to estimate the graphs $\widehat{G}^{\ell}$ in Step~3 of Algorithm~\ref{alg:huhl}.  In practice, when dealing with non-Gaussian data\footnote{In this case, a different conditional independence test needs to be used.}, $\tau$ can be chosen using methods in \cite{meinshausen2010stability}.  For the synthetic data, we use $K = 5$ (number rounds of measurements).  For real data, we use $K = 10$.

We specify $q = n \times p$, the desired number of scalar measurements, to our algorithm and obtain a matrix $\bar{X}$ of size $\bar{n} \times p$, where $\bar{n} \ge n$.  Each column in this matrix corresponds to the samples obtained from a random variable.  Since we perform active learning, some entries in $\bar{X}$ will be missing.  Moreover, in general, we may not be able to get exactly $q$ scalar measurements, so we stop obtaining measurements until the maximum possible number of measurements have been made.  Note that, when learning the final graph, we only need to consider the random variables over which we have $\bar{n}$ observations.  This is because the edges, and the non-edges, for all other random variables are estimated in the active learning component.  

We compare the active learning graphs to two other estimated graphs.  The first is the graph estimated using $n \times p$ \textit{nonactive} or \textit{passive} measurements.  The second is the graph estimated using an $\bar{n} \times p$ measurement matrix $\widetilde{X}$ that contains randomly chosen missing entries that sum to the number of missing entries in $\bar{X}$.  $\CIT$ can be easily applied to $\widetilde{X}$.  We emphasize that all three graphs, active, nonactive, and random, are estimated using roughly \textit{the same number of scalar measurements}.

We use the extended Bayesian information criterion (EBIC) for model selection \cite{foygel2010extended}.  EBIC requires an appropriate input parameter $\gamma$ that controls the sparsity of the graph.  We use $\gamma = 0.5$ as suggested by the authors in \cite{foygel2010extended}.  When using $\bar{X}$ and $\widetilde{X}$, we make appropriate modifications to compute the likelihood function, see \cite{stadler2012missing} for more details.

We use three measures to compare an estimated graph to the true graph.  The first is the true positive rate (TPR), which is the number of true edges estimated divided by the total number of true edges.  The second is the false discovery rate (FDR), which is the number of falsely detected edges divided by the number of edges estimated.  The third is the edit distance (ED), which is the number of true edges missed plus the number of falsely detected edges.

\subsection{Results}

Figure~\ref{fig:hubg}(a)-(b) show results for the chain and hub graphs with $p = 100$ and $p = 200$ vertices.  The graphs estimated in each case are oracle estimates, i.e., the true graph was used to select an optimal threshold for $\CIT$.  This allows us to quantify the  benefits of active learning and also validate our theoretical results.  The plots show the variation of the edit distance for active, nonactive, and random graph estimates as the number of scalar measurements increase.  For small $q$, no active learning is done since there are not enough measurements to separate the weak parts of the graph from the strong parts.  As $q$ increases, we clearly see the benefits of using active learning.

\begin{figure}[h]
\centering
\subfigure[Chain Graph]{
\includegraphics[scale=0.271]{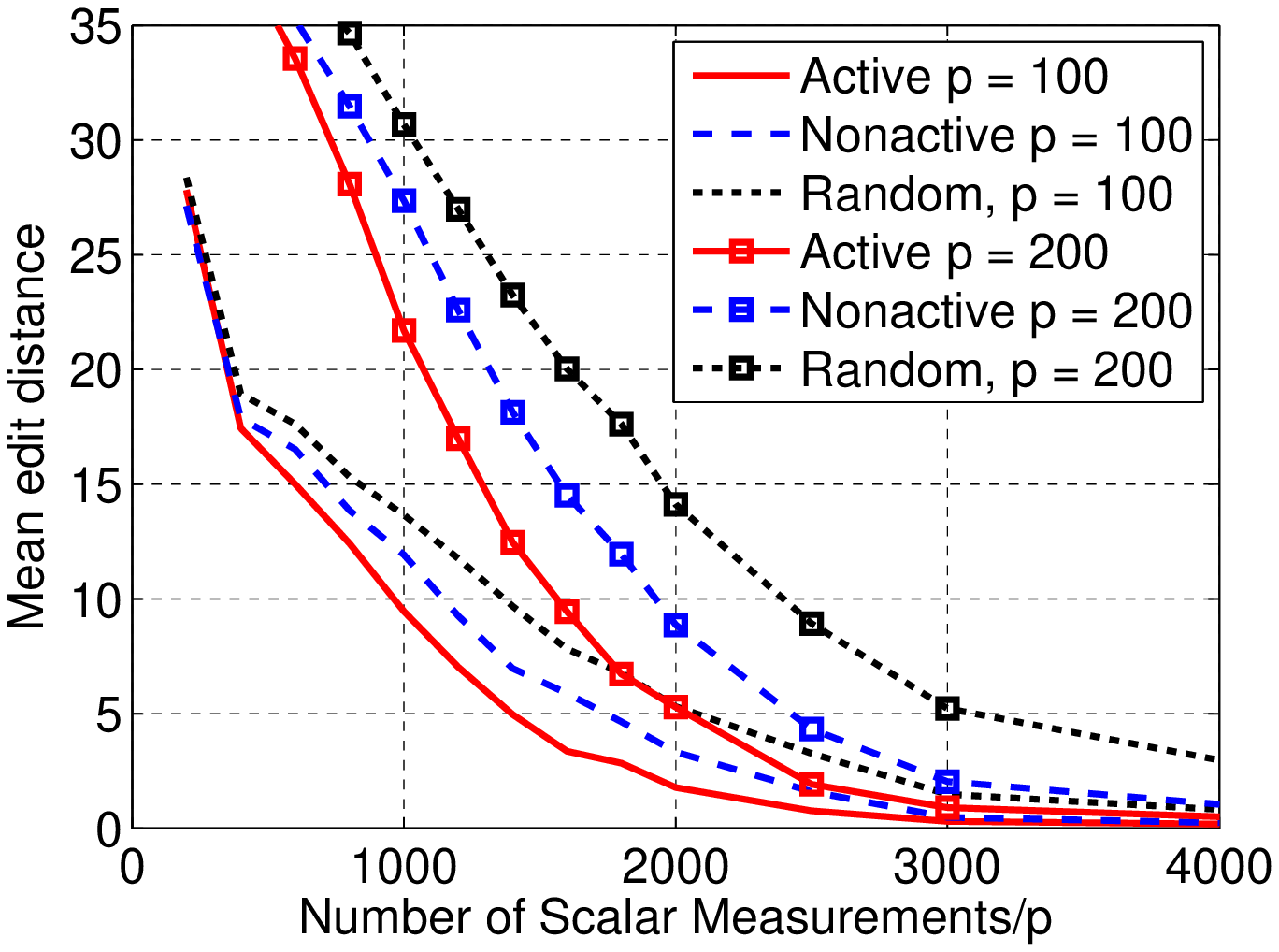}
}
\subfigure[Hub Graph]{
\includegraphics[scale=0.271]{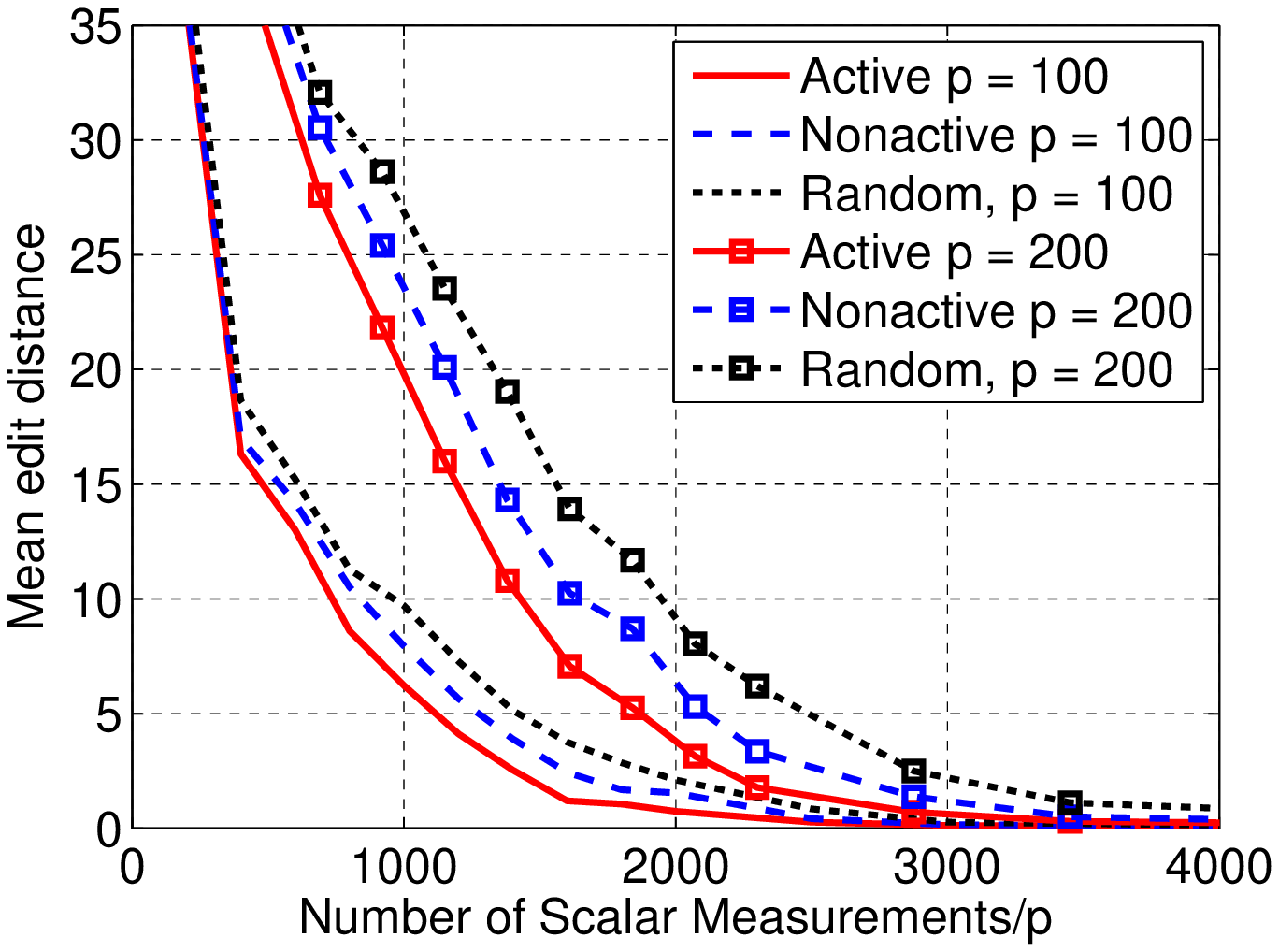}
}
\caption{Mean edit distance vs. number of scalar measurements over $50$ trials for chain and hub graphs.}
\label{fig:hubg}
\end{figure}

%\begin{table}
%\centering
%\caption{Oracle results for Scale-free graph with $p = 100$ vertices}
%\label{tab:scalefreep100}
%\input{tabscalefreep100}
%\end{table}

%
%\begin{table}
%\centering
%\caption{Oracle results for Scale-free graph with $p = 200$ vertices}
%\label{tab:scalefreep200}
%\input{tabscalefreep200}
%\end{table}

Table~\ref{tab:clusterp400} shows results for the cluster graph when $p = 400$ and $n = 200,400,600$.  Each entry in the table is the mean value of the metric over $50$ trials with the standard error given in brackets.  We present both oracle and model selection results.  In both cases, the benefits of active learning is clear.  

%Fig.~\ref{fig:clus} shows a graph of the mean number of measurements obtained from each random variable for the cluster graph with $p = 400$ and $n = 600$.  
%\begin{wrapfigure}{r}{0.25\textwidth}
%\centering
%\includegraphics[scale=0.3]{MeasCluster.eps}
%\caption{Mean number of measurements from each vertex for the cluster graph when $p = 400$ and $n = 600$.}
%\label{fig:clus}
%\end{wrapfigure}
%The active learning algorithm clearly draws more measurements from the first $80$ vertices.  This empirical observation is consistent with our theory since the first $80$ vertices are associated with weak edges that are difficult to learn.  We refer to the supplementary material for results on two real world datasets.

\section{Conclusions}

We have proposed an active learning algorithm for graphical model selection by adapting measurements drawn from a graphical model to certain subsets of vertices in a graph.  We have identified a broad class of graphical models for which active learning can lead to significant savings in the total number of measurements needed for consistent graph recovery.

We highlight two interesting directions of future research.  First, our algorithm depends on successfully selecting a superset of the true graph.  Although we used a heuristic in our implementation, it will be extremely useful to have a consistent estimator that can reliably prune out several edges.  Second, it will be interesting to study active learning algorithms for parameter estimation in graphical models.

%\newpage
\section*{Acknowledgment}
Thanks to Vincent Tan and Gautam Dasarathy for valuable feedback and discussions. The work of D. Vats was supported in part by an Institute for Mathematics and Applications Postdoctoral Fellowship.

\onecolumn

\input{supplementMat}

\end{document}

%% file: figexampleactiveugms.tex
\begin{figure}
\centering
\scalebox{0.38} % Change this value to rescale the drawing.
{
\begin{pspicture}(0,-2.350563)(17.12,2.350563)
\definecolor{color0b}{rgb}{0.8,0.8,1.0}
\definecolor{color804b}{rgb}{1.0,0.8,0.6}
\psbezier[linewidth=0.04,linecolor=white,fillstyle=solid,fillcolor=color804b](14.468284,0.2505631)(15.12657,0.62941027)(15.157397,0.7305631)(16.128698,-0.110288695)(17.1,-0.9511405)(16.359856,-1.5283108)(15.68,-1.9294369)(15.000145,-2.3305628)(14.450852,-1.4709399)(14.075898,-1.119311)(13.700944,-0.767682)(13.809997,-0.128284)(14.468284,0.2505631)
\pscircle[linewidth=0.02,dimen=outer,fillstyle=solid,fillcolor=color0b](12.575501,-0.079436876){0.34999996}
\pscircle[linewidth=0.02,dimen=outer,fillstyle=solid,fillcolor=color0b](14.335502,0.8005631){0.35}
\pscircle[linewidth=0.02,dimen=outer,fillstyle=solid,fillcolor=color0b](13.5355015,0.040563125){0.34999993}
\pscircle[linewidth=0.02,dimen=outer,fillstyle=solid,fillcolor=color0b](14.295502,-0.73943657){0.35}
\pscircle[linewidth=0.02,dimen=outer,fillstyle=solid,fillcolor=color0b](13.575501,1.4605631){0.34999993}
\pscircle[linewidth=0.02,dimen=outer,fillstyle=solid,fillcolor=color0b](15.195501,0.12056313){0.35}
\pscircle[linewidth=0.02,dimen=outer,fillstyle=solid,fillcolor=color0b](12.695502,0.82056326){0.35}
\pscircle[linewidth=0.02,dimen=outer,fillstyle=solid,fillcolor=color0b](15.195501,-0.73943657){0.35}
\pscircle[linewidth=0.02,dimen=outer,fillstyle=solid,fillcolor=color0b](15.195501,-1.5794365){0.35}
\pscircle[linewidth=0.02,dimen=outer,fillstyle=solid,fillcolor=color0b](16.135502,-0.73943657){0.35}
\psline[linewidth=0.04cm](12.965502,1.0305632)(13.305502,1.2705632)
\psline[linewidth=0.04cm](13.825502,1.2505631)(14.085502,1.0305632)
\psline[linewidth=0.04cm](13.7855015,0.27056327)(14.125502,0.5705632)
\psline[linewidth=0.04cm](13.7455015,-0.18943672)(14.065502,-0.5094367)
\psline[linewidth=0.04cm](14.605502,0.6105631)(14.925502,0.33056328)
\psline[linewidth=0.04cm](14.525501,-0.4894367)(14.965502,-0.109436736)
\psline[linewidth=0.04cm](14.485501,-1.0094368)(14.925502,-1.3894368)
\psline[linewidth=0.04cm](15.445501,-0.08943678)(15.925502,-0.5094367)
\psline[linewidth=0.04cm](15.905501,-0.9494369)(15.465502,-1.3694369)
\psline[linewidth=0.04cm](14.625502,-0.72943676)(14.865502,-0.72943676)
\psline[linewidth=0.04cm](15.545502,-0.72943676)(15.805502,-0.72943676)
\psbezier[linewidth=0.04,linecolor=white,fillstyle=solid,fillcolor=color804b](0.8378859,1.6001736)(1.5481508,2.330563)(2.9255018,1.8305632)(4.1455016,0.5305631)(5.3655014,-0.76943696)(4.829667,-1.4937365)(3.9514241,-1.8615868)(3.0731812,-2.229437)(2.9110036,-1.2998265)(1.4655018,-0.90943694)(0.02,-0.5190474)(0.127621,0.86978406)(0.8378859,1.6001736)
\pscircle[linewidth=0.02,dimen=outer,fillstyle=solid,fillcolor=color0b](0.87550175,-0.03943695){0.34999996}
\pscircle[linewidth=0.02,dimen=outer,fillstyle=solid,fillcolor=color0b](2.635502,0.84056306){0.35}
\pscircle[linewidth=0.02,dimen=outer,fillstyle=solid,fillcolor=color0b](1.8355018,0.080563046){0.34999993}
\pscircle[linewidth=0.02,dimen=outer,fillstyle=solid,fillcolor=color0b](2.5955017,-0.69943666){0.35}
\pscircle[linewidth=0.02,dimen=outer,fillstyle=solid,fillcolor=color0b](1.8755018,1.500563){0.34999993}
\pscircle[linewidth=0.02,dimen=outer,fillstyle=solid,fillcolor=color0b](3.4955018,0.16056305){0.35}
\pscircle[linewidth=0.02,dimen=outer,fillstyle=solid,fillcolor=color0b](0.9955019,0.8605632){0.35}
\pscircle[linewidth=0.02,dimen=outer,fillstyle=solid,fillcolor=color0b](3.4955018,-0.69943666){0.35}
\pscircle[linewidth=0.02,dimen=outer,fillstyle=solid,fillcolor=color0b](3.4955018,-1.5394367){0.35}
\pscircle[linewidth=0.02,dimen=outer,fillstyle=solid,fillcolor=color0b](4.4355016,-0.69943666){0.35}
\psline[linewidth=0.04cm](1.2655019,1.0705632)(1.6055018,1.3105632)
\psline[linewidth=0.04cm](2.125502,1.2905631)(2.3855019,1.0705632)
\psline[linewidth=0.04cm](2.0855017,0.31056318)(2.4255018,0.61056316)
\psline[linewidth=0.04cm](2.0455017,-0.1494368)(2.365502,-0.46943682)
\psline[linewidth=0.04cm](2.9055018,0.65056306)(3.2255018,0.37056318)
\psline[linewidth=0.04cm](2.8255017,-0.44943678)(3.2655017,-0.06943681)
\psline[linewidth=0.04cm](2.7855017,-0.96943694)(3.2255018,-1.349437)
\psline[linewidth=0.04cm](3.7455018,-0.049436852)(4.2255015,-0.46943682)
\psline[linewidth=0.04cm](4.2055016,-0.90943694)(3.7655017,-1.3294369)
\psline[linewidth=0.04cm](2.9255018,-0.68943685)(3.1655018,-0.68943685)
\psline[linewidth=0.04cm](3.8455017,-0.68943685)(4.1055017,-0.68943685)
\psbezier[linewidth=0.04,linecolor=white,fillstyle=solid,fillcolor=color804b](7.687249,1.0305631)(8.447249,1.5305631)(8.707249,1.6105632)(9.892751,0.5105632)(11.078252,-0.58943677)(10.576916,-1.5137364)(9.698673,-1.8815867)(8.820431,-2.2494369)(7.747249,-1.2694368)(7.387249,-0.80943686)(7.0272493,-0.34943688)(6.927249,0.5305631)(7.687249,1.0305631)
\pscircle[linewidth=0.02,dimen=outer,fillstyle=solid,fillcolor=color0b](6.6227508,-0.059436873){0.34999996}
\pscircle[linewidth=0.02,dimen=outer,fillstyle=solid,fillcolor=color0b](8.382751,0.82056314){0.35}
\pscircle[linewidth=0.02,dimen=outer,fillstyle=solid,fillcolor=color0b](7.582751,0.060563125){0.34999993}
\pscircle[linewidth=0.02,dimen=outer,fillstyle=solid,fillcolor=color0b](8.342751,-0.7194366){0.35}
\pscircle[linewidth=0.02,dimen=outer,fillstyle=solid,fillcolor=color0b](7.6227508,1.480563){0.34999993}
\pscircle[linewidth=0.02,dimen=outer,fillstyle=solid,fillcolor=color0b](9.242751,0.14056313){0.35}
\pscircle[linewidth=0.02,dimen=outer,fillstyle=solid,fillcolor=color0b](6.742751,0.8405633){0.35}
\pscircle[linewidth=0.02,dimen=outer,fillstyle=solid,fillcolor=color0b](9.242751,-0.7194366){0.35}
\pscircle[linewidth=0.02,dimen=outer,fillstyle=solid,fillcolor=color0b](9.242751,-1.5594366){0.35}
\pscircle[linewidth=0.02,dimen=outer,fillstyle=solid,fillcolor=color0b](10.182751,-0.7194366){0.35}
\psline[linewidth=0.04cm](7.012751,1.0505632)(7.352751,1.2905632)
\psline[linewidth=0.04cm](7.872751,1.2705631)(8.132751,1.0505632)
\psline[linewidth=0.04cm](7.832751,0.29056326)(8.17275,0.59056324)
\psline[linewidth=0.04cm](7.792751,-0.16943672)(8.112751,-0.48943675)
\psline[linewidth=0.04cm](8.652751,0.63056314)(8.972751,0.35056326)
\psline[linewidth=0.04cm](8.572751,-0.4694367)(9.012751,-0.08943673)
\psline[linewidth=0.04cm](8.532751,-0.98943686)(8.972751,-1.3694369)
\psline[linewidth=0.04cm](9.492751,-0.06943677)(9.972751,-0.48943675)
\psline[linewidth=0.04cm](9.952751,-0.92943686)(9.512751,-1.3494369)
\psline[linewidth=0.04cm](8.67275,-0.7094368)(8.912751,-0.7094368)
\psline[linewidth=0.04cm](9.592751,-0.7094368)(9.852751,-0.7094368)
\psline[linewidth=0.04cm,arrowsize=0.013cm 2.0,arrowlength=1.4,arrowinset=0.4,doubleline=true,doublesep=0.08]{->}(4.42,1.2905631)(5.86,1.2905631)
\usefont{T1}{ptm}{m}{n}
\rput(5.09,1.885563){\huge Measurements}
\psline[linewidth=0.04cm,arrowsize=0.013cm 2.0,arrowlength=1.4,arrowinset=0.4,doubleline=true,doublesep=0.08]{->}(10.14,1.230563)(11.58,1.230563)
\usefont{T1}{ptm}{m}{n}
\rput(10.89,1.885563){\huge Measurements}
\end{pspicture} 
}
\vspace{-0.2cm}
\caption{ Shaded regions represent the active vertices.  As measurements are acquired, the number of active vertices decrease.}%  All edges and non-edges outside the shaded region are edges and non-edges estimated to be in the true graph.}
\vspace{-0.2cm}
\label{fig:exampleactiveugms}
\end{figure}

%% file: FindActiveFigure.tex
\scalebox{0.48} % Change this value to rescale the drawing.
{
\begin{pspicture}(0,-1.2399999)(16.55,1.2399999)
\definecolor{color0b}{rgb}{0.8,0.8,1.0}
\psline[linewidth=0.04cm](12.2,-0.3200001)(12.64,-0.70000017)
\pscircle[linewidth=0.02,dimen=outer,fillstyle=solid,fillcolor=color0b](12.91,-0.88999987){0.35}
\psline[linewidth=0.04cm](14.034498,-0.060000125)(16.28,-0.059999995)
\psline[linewidth=0.04cm](7.2744985,-0.8800001)(8.974499,-0.22000013)
\psline[linewidth=0.04cm](6.3344984,-0.020000126)(10.4,-0.019999996)
\pscircle[linewidth=0.02,dimen=outer,fillstyle=solid,fillcolor=color0b](9.150001,-0.019999836){0.35}
\pscircle[linewidth=0.02,dimen=outer,fillstyle=solid,fillcolor=color0b](6.13,-0.009999835){0.35}
\pscircle[linewidth=0.02,dimen=outer,fillstyle=solid,fillcolor=color0b](0.35,0.030000165){0.35}
\pscircle[linewidth=0.02,dimen=outer,fillstyle=solid,fillcolor=color0b](1.2500001,-0.8099999){0.35}
\psline[linewidth=0.04cm](0.58,0.28000003)(1.02,0.66)
\psline[linewidth=0.04cm](0.54,-0.24000011)(0.9800001,-0.6200002)
\psline[linewidth=0.04cm](1.5000001,0.67999995)(1.9799998,0.26)
\psline[linewidth=0.04cm](1.9599999,-0.18000011)(1.52,-0.6000001)
\psline[linewidth=0.04cm](2.2744982,0.019999875)(4.62,0.020000005)
\pscircle[linewidth=0.02,dimen=outer,fillstyle=solid,fillcolor=color0b](3.3700001,0.020000165){0.35}
\pscircle[linewidth=0.02,dimen=outer,fillstyle=solid,fillcolor=color0b](4.54,0.020000165){0.35}
\pscircle[linewidth=0.02,dimen=outer,fillstyle=solid,fillcolor=color0b](2.1899998,0.020000165){0.35}
\pscircle[linewidth=0.02,dimen=outer,fillstyle=solid,fillcolor=color0b](1.2500001,0.88999987){0.35}
\usefont{T1}{ptm}{m}{n}
\rput(1.2144983,0.90499985){\Large 1}
\usefont{T1}{ptm}{m}{n}
\rput(0.3344983,0.044999875){\Large 2}
\usefont{T1}{ptm}{m}{n}
\rput(1.2644984,-0.79500014){\Large 3}
\usefont{T1}{ptm}{m}{n}
\rput(2.1744983,0.024999874){\Large 4}
\usefont{T1}{ptm}{m}{n}
\rput(3.3544984,0.024999874){\Large 5}
\usefont{T1}{ptm}{m}{n}
\rput(4.5244985,0.024999874){\Large 6}
\pscircle[linewidth=0.02,dimen=outer,fillstyle=solid,fillcolor=color0b](7.03,-0.8499999){0.35}
\psline[linewidth=0.04cm](6.36,0.24000004)(6.8,0.62)
\psline[linewidth=0.04cm](6.32,-0.28000012)(6.76,-0.66000015)
\psline[linewidth=0.04cm](7.28,0.64)(7.7599998,0.22)
\psline[linewidth=0.04cm](7.74,-0.22000012)(7.3,-0.64000005)
\pscircle[linewidth=0.02,dimen=outer,fillstyle=solid,fillcolor=color0b](10.32,-0.019999836){0.35}
\pscircle[linewidth=0.02,dimen=outer,fillstyle=solid,fillcolor=color0b](7.03,0.8499999){0.35}
\usefont{T1}{ptm}{m}{n}
\rput(6.9944983,0.8649999){\Large 1}
\usefont{T1}{ptm}{m}{n}
\rput(7.0444984,-0.8350001){\Large 3}
\usefont{T1}{ptm}{m}{n}
\rput(10.304499,-0.015000125){\Large 6}
\pscircle[linewidth=0.02,dimen=outer,fillstyle=solid,fillcolor=color0b](7.97,-0.019999836){0.35}
\usefont{T1}{ptm}{m}{n}
\rput(7.9544983,0.004999874){\Large 4}
\usefont{T1}{ptm}{m}{n}
\rput(6.114498,0.004999874){\Large 2}
\usefont{T1}{ptm}{m}{n}
\rput(9.134499,-0.015000125){\Large 5}
\pscircle[linewidth=0.02,dimen=outer,fillstyle=solid,fillcolor=color0b](15.03,-0.059999835){0.35}
\pscircle[linewidth=0.02,dimen=outer,fillstyle=solid,fillcolor=color0b](12.01,-0.049999837){0.35}
\psline[linewidth=0.04cm](12.24,0.20000005)(12.68,0.58000004)
\pscircle[linewidth=0.02,dimen=outer,fillstyle=solid,fillcolor=color0b](16.2,-0.059999835){0.35}
\pscircle[linewidth=0.02,dimen=outer,fillstyle=solid,fillcolor=color0b](12.91,0.8099999){0.35}
\usefont{T1}{ptm}{m}{n}
\rput(12.874498,0.82499987){\Large 1}
\usefont{T1}{ptm}{m}{n}
\rput(12.924499,-0.8750001){\Large 3}
\usefont{T1}{ptm}{m}{n}
\rput(16.184498,-0.055000126){\Large 6}
\pscircle[linewidth=0.02,dimen=outer,fillstyle=solid,fillcolor=color0b](13.849999,-0.059999835){0.35}
\usefont{T1}{ptm}{m}{n}
\rput(13.834498,-0.035000127){\Large 4}
\usefont{T1}{ptm}{m}{n}
\rput(11.994498,-0.035000127){\Large 2}
\usefont{T1}{ptm}{m}{n}
\rput(15.014499,-0.055000126){\Large 5}
\usefont{T1}{ptm}{m}{n}
\rput(3.0144982,-0.8850001){\LARGE $ G^*$}
\usefont{T1}{ptm}{m}{n}
\rput(9.204498,-0.92500013){\LARGE $ H^+$}
\usefont{T1}{ptm}{m}{n}
\rput(14.924499,-0.90500015){\LARGE $ H^-$}
\end{pspicture} 
}

%% file: FirstFigExample.tex
% Generated with LaTeXDraw 2.0.8
% Thu Oct 17 15:46:26 CDT 2013
% \usepackage[usenames,dvipsnames]{pstricks}
% \usepackage{epsfig}
% \usepackage{pst-grad} % For gradients
% \usepackage{pst-plot} % For axes
\scalebox{0.77} % Change this value to rescale the drawing.
{
\begin{pspicture}(0,-1.1695508)(10.45,1.1295508)
\definecolor{color0b}{rgb}{0.6,0.6,1.0}
\definecolor{color937b}{rgb}{0.6,0.6,1.0}
\definecolor{color1098b}{rgb}{0.8,0.8,0.8}
\psline[linewidth=0.02cm](1.58,0.4695508)(3.02,0.4695508)
\psline[linewidth=0.02cm](1.0,0.7895508)(1.36,0.5295508)
\psline[linewidth=0.02cm](0.8,0.7895508)(0.44,0.5295508)
\psline[linewidth=0.02cm](0.42,0.38955078)(0.8,0.16955078)
\pscircle[linewidth=0.04,linecolor=color0b,dimen=outer,fillstyle=solid,fillcolor=color0b](0.9,0.8495508){0.12}
\pscircle[linewidth=0.04,linecolor=color0b,dimen=outer,fillstyle=solid,fillcolor=color0b](0.9,0.12955078){0.12}
\pscircle[linewidth=0.04,linecolor=color0b,dimen=outer,fillstyle=solid,fillcolor=color0b](0.34,0.4695508){0.12}
\psline[linewidth=0.02cm](1.36,0.4095508)(1.0,0.18955079)
\pscircle[linewidth=0.04,linecolor=color0b,dimen=outer,fillstyle=solid,fillcolor=color0b](1.46,0.4695508){0.12}
\pscircle[linewidth=0.04,linecolor=color0b,dimen=outer,fillstyle=solid,fillcolor=color0b](2.02,0.4695508){0.12}
\pscircle[linewidth=0.04,linecolor=color0b,dimen=outer,fillstyle=solid,fillcolor=color0b](2.58,0.4695508){0.12}
\pscircle[linewidth=0.04,linecolor=color0b,dimen=outer,fillstyle=solid,fillcolor=color0b](3.14,0.4695508){0.12}
\psellipse[linewidth=0.02,linestyle=dashed,dash=0.16cm 0.16cm,dimen=outer](0.87,0.47955078)(0.87,0.61)
\psellipse[linewidth=0.02,linestyle=dashed,dash=0.16cm 0.16cm,dimen=outer](2.31,0.4595508)(1.11,0.29)
\usefont{T1}{ptm}{m}{n}
\rput(1.6354492,-0.9054492){\large (a)}
\psline[linewidth=0.02cm](5.02,0.8295508)(4.78,0.44955078)
\psline[linewidth=0.02cm](6.04,0.34955078)(6.52,0.4695508)
\psline[linewidth=0.02cm](6.1,0.74955076)(6.02,0.38955078)
\pscircle[linewidth=0.04,linecolor=color0b,dimen=outer,fillstyle=solid,fillcolor=color0b](6.14,0.8295508){0.12}
\pscircle[linewidth=0.04,linecolor=color0b,dimen=outer,fillstyle=solid,fillcolor=color0b](6.54,0.4695508){0.12}
\pscircle[linewidth=0.04,linecolor=color0b,dimen=outer,fillstyle=solid,fillcolor=color0b](6.02,0.32955077){0.12}
\pscircle[linewidth=0.04,linecolor=color0b,dimen=outer,fillstyle=solid,fillcolor=color0b](5.04,0.8495508){0.12}
\psline[linewidth=0.02cm](4.78,0.4295508)(5.08,0.08955079)
\pscircle[linewidth=0.04,linecolor=color0b,dimen=outer,fillstyle=solid,fillcolor=color0b](5.08,0.08955079){0.12}
\psline[linewidth=0.02cm](4.42,0.7895508)(4.78,0.44955078)
\pscircle[linewidth=0.04,linecolor=color0b,dimen=outer,fillstyle=solid,fillcolor=color0b](4.42,0.7895508){0.12}
\psline[linewidth=0.02cm](4.76,0.4295508)(4.46,0.08955079)
\pscircle[linewidth=0.04,linecolor=color0b,dimen=outer,fillstyle=solid,fillcolor=color0b](4.46,0.08955079){0.12}
\pscircle[linewidth=0.04,linecolor=color0b,dimen=outer,fillstyle=solid,fillcolor=color0b](4.78,0.44955078){0.12}
\usefont{T1}{ptm}{m}{n}
\rput(5.455449,-0.9054492){\large (b)}
\psellipse[linewidth=0.02,linestyle=dashed,dash=0.16cm 0.16cm,dimen=outer](4.76,0.4595508)(0.72,0.67)
\psellipse[linewidth=0.02,linestyle=dashed,dash=0.16cm 0.16cm,dimen=outer](6.23,0.55955076)(0.55,0.45)
\usefont{T1}{ptm}{m}{n}
\rput(0.8479004,-0.3454492){$V_1$}
\usefont{T1}{ptm}{m}{n}
\rput(2.3624024,-0.065449215){$V_2$}
\usefont{T1}{ptm}{m}{n}
\rput(4.7479,-0.4454492){$V_1$}
\usefont{T1}{ptm}{m}{n}
\rput(6.2079,-0.08544922){$V_2$}
\pscircle[linewidth=0.04,linecolor=color937b,dimen=outer,fillstyle=solid,fillcolor=color937b](8.56,-0.070449196){0.04}
\pscircle[linewidth=0.04,linecolor=color937b,dimen=outer,fillstyle=solid,fillcolor=color937b](8.76,-0.070449196){0.04}
\pscircle[linewidth=0.04,linecolor=color937b,dimen=outer,fillstyle=solid,fillcolor=color937b](8.96,-0.070449196){0.04}
\pscircle[linewidth=0.04,linecolor=color937b,dimen=outer,fillstyle=solid,fillcolor=color937b](9.16,-0.070449196){0.04}
\pscircle[linewidth=0.04,linecolor=color937b,dimen=outer,fillstyle=solid,fillcolor=color937b](9.36,-0.070449196){0.04}
\pscircle[linewidth=0.04,linecolor=color937b,dimen=outer,fillstyle=solid,fillcolor=color937b](9.56,-0.070449196){0.04}
\pscircle[linewidth=0.04,linecolor=color937b,dimen=outer,fillstyle=solid,fillcolor=color937b](9.76,-0.070449196){0.04}
\pscircle[linewidth=0.04,linecolor=color937b,dimen=outer,fillstyle=solid,fillcolor=color937b](9.96,-0.070449196){0.04}
\pscircle[linewidth=0.04,linecolor=color937b,dimen=outer,fillstyle=solid,fillcolor=color937b](10.16,-0.070449196){0.04}
\pscircle[linewidth=0.04,linecolor=color937b,dimen=outer,fillstyle=solid,fillcolor=color937b](10.36,-0.070449196){0.04}
\psframe[linewidth=0.02,dimen=outer,fillstyle=solid,fillcolor=color1098b](10.4,0.5695508)(8.52,0.029550802)
\usefont{T1}{ptm}{m}{n}
\rput(8.0279,0.3345508){$n_0$}
\psframe[linewidth=0.02,dimen=outer,fillstyle=solid,fillcolor=color1098b](9.2,1.0695508)(8.52,0.5695508)
\usefont{T1}{ptm}{m}{n}
\rput(8.0479,0.8545508){$n_1$}
\usefont{T1}{ptm}{m}{n}
\rput(9.045449,-0.9054492){\large (c)}
\usefont{T1}{ptm}{m}{n}
\rput(8.8479,-0.46544918){$\widehat{V}_1$}
\usefont{T1}{ptm}{m}{n}
\rput(9.8279,-0.46544918){$\widehat{V}_2$}
\psline[linewidth=0.02cm](8.46,-0.070449196)(8.46,-0.1304492)
\psline[linewidth=0.02cm](8.46,-0.1304492)(9.24,-0.1304492)
\psline[linewidth=0.02cm](9.24,-0.1304492)(9.24,-0.070449196)
\psline[linewidth=0.02cm](9.08,-0.1504492)(9.08,-0.2104492)
\psline[linewidth=0.02cm](9.08,-0.2104492)(10.44,-0.2104492)
\psline[linewidth=0.02cm](10.44,-0.2104492)(10.44,-0.1504492)
\end{pspicture} 
}

%% file: tabclusterp400.tex
\begin{tabular}{lllll;{1pt/2pt}llll}
\hline
 %&&&Oracle Results& && Model Selection (using EBIC) Results & \\
 &&  \multicolumn{3}{c}{Oracle Results } &\multicolumn{3}{c}{Model Selection Results }\\
 $n$ & Alg & TPR & FDR & ED & TPR & FDR & ED \\
\hline
200 &Nonactive &$0.409$ ($0.003$) &$0.149$ ($0.003$) &$283$ ($0.969$) &$0.29$ ($0.000$) &$0.029$ ($0.000$) &$307$ ($0.145$)\\
  &Active &$0.405$ ($0.002$) &$0.120$ ($0.003$) &$278$ ($0.697$) &$0.296$ ($0.000$) &$0.022$ ($0.000$) &$303$ ($0.157$)\\
\hline
400 &Nonactive &$0.675$ ($0.002$) &$0.0726$ ($0.003$) &$162$ ($0.838$)&$0.568$ ($0.001$) &$0.0148$ ($0.000$) &$188$ ($0.199$)\\
  &Active &$0.695$ ($0.003$) &$0.0666$ ($0.002$) &$152$ ($0.765$) &$0.575$ ($0.001$) &$0.0111$ ($0.000$) &$184$ ($0.202$)\\
\hline
600 &Nonactive &$0.787$ ($0.002$) &$0.0381$ ($0.001$) &$104$ ($0.689$)& $0.739$ ($0.000$) &$0.015$ ($0.000$) &$116$ ($0.144$)\\
  &Active &$0.819$ ($0.002$) &$0.0488$ ($0.001$) &$95.5$ ($0.702$)&$0.747$ ($0.000$) &$0.001$ ($0.000$) &$111$ ($0.161$)\\
\hline
\end{tabular}

%% file: supplementMat.tex
\newpage

\appendix

\section{Selecting the Graphs $H^+$ and $H^-$ in Algorithm~5}
\label{app:a}

\begin{mdframed}
\textbf{Algorithm 6:} Find graphs $H^+$ and $H^-$
\begin{itemize}
\item  \textit{Inputs:} $\Xf$, $\widehat{E}$ and $\widehat{F}$
\item \textbf{for} $\ell = 1,\ldots,L$
\begin{itemize}
\item Estimate a graph $\widehat{G}^{(\ell)}$ using random subsamples such that
\begin{itemize}
\item $(i,j) \in E(\widehat{G}^{(\ell)})$ $\forall$ $(i,j) \in \widehat{E}$
\item $(i,j) \notin E(\widehat{G}^{(\ell)})$ $\forall$ $(i,j) \in \widehat{F}$
\end{itemize}
\end{itemize} 
\item $Q_{ij} \gets$ Fraction of times the edge $(i,j)$ appears in the graphs $\widehat{G}^{(1)},\ldots,\widehat{G}^{(L)}$ \\
\item Find $H^+$ s.t. $(i,j)\! \in\! E(H^+)$ for $Q_{ij} \!\ge \!\alpha^+$ \\
\item Find $H^-$ s.t. $(i,j)\! \in\! E(H^-)$ for $Q_{ij} \!\ge\! \alpha^-$ \\
\item \textbf{return} $H^+$ and $H^-$
\end{itemize}
\end{mdframed}

In this section, we discuss the algorithm to estimate the graphs $H^+$ and $H^-$ in Algorithm~2.  The main idea, outlined in Algorithm~6 above, is to use stability selection \cite{meinshausen2010stability} to estimate edges in the unknown graph $G^*$.  We first estimate multiple different graphs using $L$ ($30$ in our simulations) randomly subsampled measurements (Line~1).  The graphs are estimated in such a way that all edges in $\widehat{E}$ are in the estimated graph and all edges in $\widehat{F}$ are not in the estimated graph.  This is done so that the graphs estimated are consistent with prior estimates of $G^*$.  Next, for each edge $(i,j) \in V \times V$, we compute the fraction of times it appears in one of the estimated graphs.  We store all these values in the matrix $Q_{ij}$.  We choose $H^-$ so that it contains all edges for which $Q_{ij} \ge \alpha^{-}$.  We choose $H^+$ so that it contains all edges for which $Q_{ij} \ge \alpha^{+}$.  Both $\alpha^-$ and $\alpha^+$ influence the performance of the active learning algorithm.  We conservatively choose them so that $\alpha^{-} = 1.0$ and $\alpha^+ = 0.1$. 

\section{Proof of Theorem~\ref{thm:firstthm}}
\label{app:b}

In this section, we analyze Algorithm~3.  The proof methodology is motivated from \cite{KalischBuhlmann2007}.  Throughout this section, we assume that $\widehat{G} = \CIT(\Xf^n,\eta,\tau_n)$, where $\CIT$ is outlined in Algorithm~1.
We are interested in finding conditions under which $\widehat{G} = G^*$ with high probability.  To this end, define the set $B_{\eta}$ as follows
\begin{equation}
B_{\eta} = \{(i,j,S): i, j \in V, i \ne j, S \subseteq V \backslash \{i,j\}, |S| \le \eta \} \,.
\end{equation}
The following concentration inequality follows from \cite{KalischBuhlmann2007} and \cite{AnimaTanWillsky2011b}
\begin{lemma}
Under Assumptions~(A1) and (A3), there exists constants $c_1$ and $c_2$ such that for $0 < \epsilon < 1$,
\begin{equation}
\sup_{(i,j,S) \in B_{\eta}}\Pb\left(||{\rho}_{ij|S}| - |\widehat{\rho}_{ij|S}|| > \epsilon \right) \le c_1 \exp \left( - c_2 (n - \eta) \epsilon^2 \right) \,,
\label{eq:conceq}
\end{equation}
where $C_1$ is a constant, and $n$ is the number of vector valued measurements made of $X_i, X_j$, and $X_S$.
\end{lemma}
\begin{proof}
Applies Lemma~2 from \cite{KalischBuhlmann2007} to the result in Lemma~18 in \cite{AnimaTanWillsky2011b}.
\end{proof}

Let $p_e = \Pb( \widehat{G} \ne G)$, where the probability measure $\Pb$ is with respect to $P_X$.  Recall that we threshold the empirical conditional partial correlation $\widehat{\rho}_{ij|S}$ to test for conditional independence, i.e., $\widehat{\rho}_{ij|S} \le \tau_n \Longrightarrow X_i \ind X_j | X_S$.  An error may occur if there exists two distinct vertices $i$ and $j$ such that either $\rho_{ij|S} = 0$ and $|\widehat{\rho}_{ij|S}| > \lambda_n$ or $|\rho_{ij|S}| > 0$ and $|\widehat{\rho}_{ij|S}| \le \tau_n$.  Thus, we have
\begin{align}
p_e &\le \Pb({\cal E}_1) + \Pb({\cal E}_2) \,, \\
\Pb({\cal E}_1) &= \Pb\left( \bigcup_{(i,j) \notin G} \{ \text{$\exists$ $S$ s.t. $|\widehat{\rho}_{ij|S}| > \lambda_n$} \} \right) \\
\Pb({\cal E}_2) &= \Pb\left( \bigcup_{(i,j) \in G} \{ \text{$\exists$ $S$ s.t. $|\widehat{\rho}_{ij|S}| \le \lambda_n$} \} \right) \,.
\end{align}
We will find conditions under which $\Pb({\cal E}_1) \rightarrow 0$ and $\Pb({\cal E}_2) \rightarrow 0$ which will imply that $P_e \rightarrow 0$.  The term $\Pb({\cal E}_1)$, probability of including an edge in $\widehat{G}$ that does not belong to the true graph, can be upper bounded as follows:
\begin{align}
\Pb({\cal E}_1) &\le \Pb\left( \bigcup_{(i,j) \notin G} \{ \text{$\exists$ $S$ s.t. $|\widehat{\rho}_{ij|S}| > \tau_n$} \} \right)
\le \Pb\left( \bigcup_{(i,j) \notin G, S \subset V \backslash \{i,j\}} \{ \text{$|\widehat{\rho}_{ij|S}| > \tau_n$} \} \right) \\
&\le p^{\eta+2} \sup_{(i,j,S) \in B_{\eta}} \Pb\left(|\widehat{\rho}_{ij|S}| > \tau_n\right) \\
&\le c_1 p^{\eta+2} \exp\left(-c_2 (n-\eta)\tau_n^2 \right)
= c_1 \exp\left( (\eta+2) \log(p) -c_2 (n-\eta)\tau_n^2 \right)
\end{align}
The terms $p^{\eta+2}$ comes from the fact that there are at most $p^2$ number of edges and the algorithm searches over at most $p^{\eta}$ number of separators for each edge.  Choosing $\tau_n$ such that $\tau_n > c_1 (\eta+2) \log p / (n-\eta)$ ensures that $\Pb({\cal E}_1) \rightarrow 0$ as $p \rightarrow \infty$.  

Suppose we select $\tau_n < c_3 \rho_{\min}$ for a constant $c_3 < 1$.
The term $\Pb({\cal E}_2)$, probability of not including an edge in $\widehat{G}$ that does belong to the true graph, can be upper bounded as follows:
\begin{align}
\Pb({\cal E}_2) &\le \Pb\left( \bigcup_{(i,j) \in G} \{ \text{$\exists$ $S$ s.t. $|\widehat{\rho}_{ij|S}| \le \tau_n$} \} \right) \\
&\le \Pb\left( \bigcup_{(i,j) \in G, S \subset V \backslash \{i,j\}}  \text{$|\rho_{ij|S}|-|\widehat{\rho}_{ij|S}| > |\rho_{ij|S}| - \tau_n$} \right) \\
&\le p^{\eta+2} \sup_{(i,j,S) \in B_{\eta}} \Pb\left(|\rho_{ij|S}|-|\widehat{\rho}_{ij|S}| > |\rho_{ij|S}| - \tau_n\right) \\
&\le p^{\eta+2} \sup_{(i,j,S) \in B_{\eta}} \Pb\left(||\rho_{ij|S}|-|\widehat{\rho}_{ij|S}|| > \rho_{\min} - \tau_n\right) \\
&\le c_1 p^{\eta+2} \exp\left(-c_2 (n-\eta)(\rho_{\min} - \tau_n)^2 \right)
= c_1 \exp\left( (\eta+2) \log(p) - c_4(n-\eta) \rho_{\min}^2 \right) \label{eq:a5}
\,.
\end{align}
To obtain (\ref{eq:a5}), we use the choice of $\tau_n$ so that $(\rho_{\min}-\tau_n) > (1-c_3) \rho_{\min}$.  For an appropriate constant $c_5 > 0$, choosing
$
n > \eta + c_5 \rho_{\min}^{-2} (\eta+2) \log(p)$ ensures $\Pb({\cal E}_2) \rightarrow 0$ as $n,p \rightarrow \infty$.  We note that the choice of $c_5$ only depends on $M$.  This concludes the proof. $\qed$

\section{Proof of Theorem~\ref{thm:firstthmA}}
\label{app:c}
In this section, we analyze Algorithm~4.  Recall the assumption that there exists sets of vertices $V_1$, $V_2$, and $T$ such that there are no edges between $V_1 \backslash T$ and $V_2 \backslash T$ in $G^*$.  Note that we only assume the \textit{existence} of these clusters and the corresponding graph decomposition.  Now, let $\widehat{G}$ be the graph estimated after Step~2 of Algorithm~4, i.e., after drawing $n_0$ measurements.  Define the event ${\cal D}$ as
\[
{\cal D} = \{ \widehat{G}[V_2] = G^*[V_2]  \text{ and } \forall \; i \in V_1\backslash T \text{ and } \forall \; j \in V_2\backslash T, \; (i,j) \notin \widehat{G} \} \,.
\]
In words, ${\cal D}$ defines the event that after $n_0$ measurements, the $\CIT$ algorithm is able to accurately identify all the edges and the non-edges over $V_2$ and all the non-edges that connect $V_1$ and $V_2$.  Given that ${\cal D}$ is true, it is easy to see that for any two-cluster decomposition of $\widehat{G}$ over clusters $\widehat{V}_1$ and $\widehat{V}_2$ such that $\widehat{G}[\widehat{V}_2] = G^*[\widehat{V}_2]$, we have that $\widehat{V}_1 \subseteq V_1$ and $V_2 \subseteq \widehat{V}_2$.

Let $\widehat{G}_1$ be the graph estimated in Step~5 of Algorithm~4 and let $\widehat{G}_F = \widehat{G}[\widehat{V}_2] \cup \widehat{G}_1$ be the output of Algorithm~4.  Conditioning on ${\cal D}$, and using the assumption that $\widehat{V}_1 = V_1$, we have
\begin{align}
\Pb(\widehat{G}_F \ne G^*) &= \Pb(\widehat{G}_F \ne G^* | {\cal D}) \Pb({\cal D}) + 
\Pb(\widehat{G}_F \ne G^* | {\cal D}^c ) \Pb({\cal D}^c) \\
 &\le  \Pb(\widehat{G}_F \ne G^* | {\cal D}) +  \Pb({\cal D}^c) \,.
\end{align}
We now make use of Theorem~\ref{thm:firstthm}.  Given the scaling of $n_0$, it is clear that $\Pb({\cal D}^c) \rightarrow 0$ as $p \rightarrow \infty$.  Furthermore, given that ${\cal D}$ holds, we only need to estimate the edges over $\widehat{V}_1$ in Step~5 of Algorithm~4.  Since $\widehat{V}_1 \subseteq V_1$ when given ${\cal D}$, it follows that $ \Pb(\widehat{G}_F \ne G^* | {\cal D}) \rightarrow 0$ as $p \rightarrow \infty$ given the scaling of $n_0 + n_1$.  This concludes the proof. $\qed$

\section{Proof of Theorem~\ref{thm:lb}}
\label{app:d}

Once $V_1$ and $V_2$ have been identified, by the global Markov property of graphical models, the graph learning can be decomposed into two independent problems of learning the edges in $G[V_1]$ and learning the edges in $G[V_2]$.  Thus, $p_e(\psi)$ can be lower bounded by
\[
\max_{\Theta(G) } \left\{ \max\left[ \Pb (\psi(\Xf_{V_1}^n) \ne G[V_1]), \Pb(\psi(\Xf_{V_2}^n) \ne G[V_2]) \right] \right\} \,,
\]
where $\Theta(G) \in {\cal G}_{p,p_1,p_2,\eta,d}(\theta_1,\theta_2)$.  By definition, we know that $G[V_1]$ is sampled uniformly from ${\cal G}_{p_1,\eta,d}$.  By identifying that ${\cal G}_{p_1,0,d} \subseteq {\cal G}_{p_1,\eta,d}$, we can now make use of the results in \cite{Wainwright2009} for degree bounded graphs.  In particular, we have from \cite{Wainwright2009} that if
\[
n \le \max\left\{ \frac{\log \binom{p_1-d}{2} -1 }{4 \theta_1^2},
\frac{\log \binom{p_2-d}{2} -1 }{4 \theta_2^2} \right\} \,,
\]
then $p_e(\psi) \rightarrow 1$ as $n \rightarrow \infty$.  This leads to the necessary condition in the statement of the theorem and concludes the proof. $\qed$

\section{Numerical Results on Scale-Free Graphs}
\label{app:e}
Tables~\ref{tab:scalefreep400} shows results for scale-free graphs.  It is typical for scale-free graphs to contain a small number of vertices that act as hubs and are connected to many other vertices in the graph.   The inverse covariance is constructed as in the Hub graph case.  For this graphical model, the weak edges correspond to all edges that connect to vertices with high degree.  We again see that active learning results in superior performance than passive learning.

\begin{table}[h]
\centering
\caption{Scale-free graph with $p = 400$ vertices}
\label{tab:scalefreep400}
{\small{
\input{tabscalefreep400}}}
\end{table}

%% file: tabscalefreep400.tex
\begin{tabular}{lllll;{1pt/2pt}llll}
\hline
 %&&&Oracle Results& && Model Selection (using EBIC) Results & \\
 &&  \multicolumn{3}{c}{Oracle Results } &\multicolumn{3}{c}{Model Selection Results }\\
 $n$ & Alg & TPR & FDR & ED & TPR & FDR & ED \\
\hline
200 &Nonactive  &$0.405$ ($0.001$) &$0.059$ ($0.002$) &$247$ ($0.410$)
&$0.382$ ($0.000$) &$0.033$ ($0.000$) &$251$ ($0.121$)\\
  &Active &$0.422$ ($0.001$) &$0.040$ ($0.002$) &$237$ ($0.402$)
 &$0.391$ ($0.000$) &$0.017$ ($0.000$) &$245$ ($0.100$)\\
\hline
400 &Nonactive  &$0.522$ ($0.002$) &$0.043$ ($0.002$) &$200$ ($0.341$)
&$0.500$ ($0.000$) &$0.023$ ($0.000$) &$204$ ($0.107)$ \\
  &Active &$0.545$ ($0.001$) &$0.036$ ($0.002$) &$189$ ($0.340$)
  &$0.509$ ($0.000$) &$0.001$ ($0.000$) &$197$ ($0.121$)\\
\hline
600 &Nonactive  &$0.605$ ($0.001$) &$0.0346$ ($0.001$) &$166$ ($0.361$) 
&$0.582$ ($0.000$) &$0.021$ ($0.000$) &$171$ ($0.123$)\\
  &Active  &$0.634$ ($0.001$) &$0.0321$ ($0.001$) &$154$ ($0.378$) 
  &$0.592$ ($0.000$) &$0.008$ ($0.000$) &$164$ ($0.131$)\\
\hline
\end{tabular}

%
%\begin{tabular}{llllll}
%\hline
% $n$ & Alg  & TP & FDR & ED \\
%\hline
%200 &Nonactive  &$0.382$ ($0.00028$) &$0.0329$ ($0.000334$) &$251$ ($0.121$)\\
%  &Active, $K = 5$  &$0.391$ ($0.000267$) &$0.0172$ ($0.000187$) &$245$ ($0.0999$)\\
%\hline
%400 &Nonactive &$0.5$ ($0.000293$) &$0.0227$ ($0.000225$) &$204$ ($0.107$)\\
%  &Active, $K = 5$ &$0.509$ ($0.000321$) &$0.00983$ ($0.000163$) &$197$ ($0.121$)\\
%\hline
%600 &Nonactive &$0.582$ ($0.00033$) &$0.0206$ ($0.000194$) &$171$ ($0.123$)\\
%  &Active, $K = 5$ &$0.592$ ($0.000343$) &$0.00815$ ($0.000141$) &$164$ ($0.131$)\\
%\hline
%\end{tabular}